\documentclass{sage}
\usepackage{makeidx}         
\usepackage{graphicx}        
\usepackage{multicol}        
\usepackage[bottom]{footmisc}

\usepackage{hyperref}

\makeindex             
                       
\usepackage{titlesec}
     
\usepackage{natbib}
\usepackage{multicol}

\usepackage[]{graphicx}
\usepackage{enumerate}
\usepackage{subfigure}
\usepackage{amssymb}
\usepackage{amsmath}
\usepackage{amsthm,url}
\newtheorem{definition}{Definition}
\newtheorem{theorem}{Theorem}
\newtheorem{remark}{Remark}

\graphicspath{{./fig/}}
\usepackage{color}
\usepackage{amsmath}
\usepackage{amssymb}
\usepackage{graphicx}
\usepackage{comment,xspace}
\usepackage{fancybox}

\newcommand{\real}{{\mathbb{R}}}
\newcommand{\reals}{\real}

\renewcommand{\natural}{{\mathbb{N}}}
\newcommand{\naturals}{\natural}

\newcommand{\xfree}{\mathcal X_{\text{free}}}
\newcommand{\xobs}{\mathcal X_{\text{obs}}}
\newcommand{\xgoal}{\mathcal X_{\text{goal}}}
\newcommand{\xinit}{x_{\mathrm{init}}}
\newcommand{\FMT}{$\text{FMT}^*\, $}
\newcommand{\PRM}{$\text{gPRM}\,$}

\newcommand{\RRTstar}{RRT$^{\ast}\,$}
\newcommand{\RRTsharp}{RRT$^{\#}\,$}
\newcommand{\PRMstar}{PRM$^\ast\,$}
\newcommand{\gprm}{\gamma_{\mathrm{PRM}}}

\usepackage{algorithmic}
\usepackage{algorithm}

\begin{document}

\title{Deterministic Sampling-Based Motion Planning: Optimality, Complexity, and Performance}

\author{Lucas Janson}
\address{Department of Statistics, Stanford University, USA}
\author{Brian Ichter}
\address{Department of Aeronautics and Astronautics, Stanford University, USA}
\author{Marco Pavone\thanks{Corresponding author; email: pavone@stanford.edu\\
This work was originally presented at the 17th International Symposium on Robotics Research, ISRR 2015. This extended and revised version includes a novel section on extensions of this work to $k$-nearest-neighbor algorithms, non-PRM algorithms, non-lattice sampling, non-uniform sampling, and kinodynamic motion planning. It also includes a significantly extended simulation section.}}
\address{Department of Aeronautics and Astronautics, Stanford University, USA}
%
%
\maketitle

\abstract{Probabilistic sampling-based algorithms, such as the probabilistic roadmap (PRM) and the rapidly-exploring random tree (RRT) algorithms,   represent one of the most successful approaches to robotic motion planning, due to their strong theoretical properties (in terms of probabilistic completeness or even asymptotic optimality) and remarkable practical performance. Such algorithms are probabilistic in that they compute a path by connecting independently and identically distributed (i.i.d.) random points in the configuration space. Their randomization aspect, however, makes several tasks challenging, including certification for safety-critical applications and use of offline computation to improve real-time execution. Hence, an important open question is whether similar (or better) theoretical guarantees and practical performance could be obtained by considering  deterministic, as opposed to random sampling sequences. The objective of this paper is to provide a rigorous answer to this question. 
Specifically, we first show that PRM, for a certain selection of
tuning parameters and deterministic low-dispersion sampling sequences, is
\emph{deterministically} asymptotically optimal, i.e., it returns a
path whose cost converges deterministically to the optimal one as the
number of points goes to infinity. Second, we characterize the
convergence rate, and we find that the factor of sub-optimality can be
very explicitly upper-bounded in terms of the  $\ell_2$-dispersion of
the sampling sequence and the connection radius of PRM. Third, we show
that an asymptotically optimal version of PRM exists with
computational and space complexity arbitrarily close to $O(n)$ (the
theoretical lower bound), where $n$ is the number of points in the
sequence. This is in stark contrast to the $O(n\, \log n)$ complexity
results for existing asymptotically-optimal probabilistic planners.
Fourth, we show that our theoretical results and insights extend to other batch-processing algorithms such as \FMT\!,  to non-uniform sampling strategies, to $k$-nearest-neighbor implementations, and to differentially-constrained problems.
Finally, through numerical experiments,  
we show that planning with deterministic low-dispersion sampling generally provides superior performance in terms of path cost and success rate. 
}

\section{Introduction}

\emph{Probabilistic} sampling-based algorithms represent a particularly successful approach to robotic motion planning problems  \citep{ST-WB-DF:05, SL:06}. The key idea behind probabilistic sampling-based
algorithms is to avoid the explicit construction of the configuration
space (which can be prohibitive in complex planning problems) and
instead conduct a search that probabilistically probes the
configuration space with independently and identically distributed (i.i.d.) random samples. This probing is enabled by
a collision detection module, which the motion planning algorithm
considers as a ``black box" \citep{SL:06}. Examples, roughly in chronological order, include
the probabilistic roadmap algorithm (PRM)  \citep{LEK-PS-JCL-MHO:96}, expansive space trees (EST) \citep{DH-JCL-RM:99a,
  JMP-NB-LEK:04}, Lazy-PRM  \citep{RB-LK:00},  the rapidly
exploring random trees algorithm (RRT)  \citep{SML-JJK:01}, sampling-based roadmap of trees (SRT)
\citep{EP-KEB-BYC-ea:05}, rapidly-exploring roadmap \citep{RA-SP-AD:11}, \PRMstar and \RRTstar
\citep{SK-EF:11},  \RRTsharp  \citep{OA-PT:13}, and the fast marching tree algorithm (\FMT\!) \citep{LJ-ES-AC-ea:15}. 
A central result is that these algorithms provide \emph{probabilistic completeness} guarantees in the sense that the probability that the planner fails to return a solution, if one exists, decays to zero as the number of samples approaches infinity \citep{JB-LK-RM-ea:00}. Recently, it has been proven that  \RRTstar\!, \PRMstar\!, \RRTsharp\!, and \FMT are asymptotically optimal, i.e., the cost of the returned solution converges almost surely to the optimum \citep{SK-EF:11,OA-PT:13,LJ-ES-AC-ea:15}.

It is natural to wonder whether the theoretical guarantees and practical performance of sampling-based algorithms would hold if these algorithms were to be de-randomized, i.e., run on a \emph{deterministic}, as opposed to random sampling sequence. This is an important question, as de-randomized planners would significantly simplify the certification process (as needed for safety-critical applications), enable the use of offline computation (particularly important for planning under differential constraints or in high-dimensional spaces---exactly the regime for which sampling-based planners are designed), and, in the case of lattice sequences, drastically simplify a number of operations (e.g., locating nearby samples). This question has received relatively little attention in the literature. 
Specifically, previous research \citep{MSB-SML-KO-ea:01, SML-MB-SRL:04, DH-JCL-HK:06} has focused on the performance of de-randomized versions of sampling-based planners in terms of convergence to feasible paths. 
A number of deterministic variants of the PRM algorithm were shown to be resolution complete (i.e., provably converging to a feasible solution as $n\to \infty$) and, perhaps surprisingly,  offer superior performance on an extensive set of numerical experiments \citep{MSB-SML-KO-ea:01, SML-MB-SRL:04}. 
Prompted by these results, a number of deterministic low-dispersion, incremental sequences have been specifically tailored to motion planning problems \citep{AY-SML:04, SRL-AY-SML:05, AY-SJ-SML-ea:09}.

The results in \citep{MSB-SML-KO-ea:01, SML-MB-SRL:04, DH-JCL-HK:06}
are restricted to convergence to feasible, as opposed to
\emph{optimal} paths. Several questions are still open. Are there
advantages of i.i.d. sampling in terms of convergence to an optimal
path? Can convergence rate guarantees for the case of deterministic
sampling be provided, similar to what is done  for probabilistic
planners in \citep{LJ-ES-AC-ea:15, AD-GM-KEB:15}? For a given number
of samples, are there advantages in terms of computational and space
complexity? The objective of this paper is to rigorously address these
questions. Our focus is on the PRM algorithm. However, we show that
similar results hold for many of the existing batch (i.e., not anytime) algorithms, including Lazy-PRM and \FMT\!.

\emph{Statement of Contributions:} The contributions of this paper are as follows. 
\begin{description}
\item[\bf Deterministic asymptotic optimality of sampling-based
  planning:] We show that the PRM algorithm is asymptotically optimal
  when run on \emph{deterministic} sampling sequences in $d$ dimensions whose
  $\ell_2$-dispersion is upper-bounded by $\gamma\, n^{-1/d}$, for
  some $\gamma \in \reals_{>0}$ (we refer to such sequences as
  deterministic low-dispersion sequences), and with a connection
  radius $r_n \in \omega( n^{-1/d})$\footnote{For $f, g:\naturals \to
    \reals$, we say $f\in O(g)$ if there exists $n_0\in
      \naturals$ and
    $k\in \reals_{>0}$ such that $|f(n)|\leq k\, |g(n)|$ for all
    $n\geq n_0$. We say $f\in \Omega(g)$ if there exists $n_0\in
    \naturals$ and $k\in \reals_{>0}$ such that $|f(n)|\geq k\,
    |g(n)|$ for all $n\geq n_0$. Finally, we say $f\in \omega(g)$ if
    $\lim_{n\to \infty} \, f(n)/g(n) = \infty$.}. In other words, the
  cost of the solution computed over $n$ samples converges
  deterministically to the optimum as $n \to \infty$. As a comparison,
  the analogue result for the case of i.i.d. random sampling holds
  almost surely or in probability \citep{SK-EF:11, LJ-ES-AC-ea:15} (as opposed to deterministically) and requires a connection radius $\Omega\left (( \log(n)/n )^{1/d} \right)$, i.e., bigger. 
\item[\bf Convergence rate:] We show that, in the absence of obstacles, the factor of sub-optimality of PRM is upper-bounded by $2D_n/(r_n-2D_n)$, where $D_n$ is the $\ell_2$-dispersion of the sampling sequence. A slightly more sophisticated result holds for the obstacle-cluttered case.  As a comparison, the analogue result for the case of i.i.d. sampling only holds in probability and is much more involved (and less interpretable) \citep{LJ-ES-AC-ea:15}.  Our results could be instrumental to the certification of sampling-based planners.
\item[\bf Computational and space complexity:] We prove that PRM, when
  run on a low-dispersion  lattice, has computational and space
  complexity $O(n^2 \, r_n^d)$. As asymptotic optimality can be
  obtained using $r_n\in \omega( n^{-1/d})$, there exists an
  asymptotically optimal version of PRM with computational and space
  complexity $\omega(n)$, where $O(n)$ represents the theoretical lower bound (as, at the very least, $n$ operations need to be carried out to load samples into memory). As a comparison, the analogous complexity results for the case of i.i.d. sampling are of order $O(n \, \log(n))$ \citep{SK-EF:11}.  
  \item[\bf Extensions:] We extend the contributions in all three of
    the preceding categories to much broader settings. Specifically,
    we find that many of the results that hold for PRM run on a
    low-dispersion lattice hold either exactly or approximately for
    $k$-nearest-neighbor algorithms, for
    other batch-processing algorithms such as \FMT\!, for non-lattice
    low-dispersion sampling such as the Halton sequence, for non-uniform sampling, and for
    kinodynamic planning.
  
\item[\bf Experimental performance:] Finally, we compare performance
  (in terms of path cost and success rates) of deterministic low-dispersion sampling
  versus i.i.d. sampling on a variety of test cases ranging
  from two to eight dimensions and including geometric, kinematic
  chain, and kinodynamic planning problems. 
   In all our examples, for a given
  number of samples, deterministic low-dispersion sampling  performs
  no worse and sometimes substantially better than  i.i.d. sampling (this is not even accounting for the potential significant speed-ups in runtime, e.g., due to fast nearest-neighbor indexing). 
\end{description}

The key insight behind our theoretical results (e.g., smaller required
connection radius, better complexity, etc.) is the factor difference
in dispersion between deterministic low-dispersion sequences versus
i.i.d. sequences, namely $O(n^{-1/d})$ versus $O((\log n)^{1/d}
\, n^{-1/d})$ \citep{PD:83, HN:92}. Interestingly, the same
$O(n^{-1/d})$  dispersion can be achieved with
non-i.i.d. \emph{random} sequences, e.g., randomly rotated and offset
lattices. As we will show, these sequences enjoy the same
\emph{deterministic} performance guarantees of deterministic
low-dispersion sequences and retain many of the benefits of
deterministic sampling (e.g., fast nearest-neighbor
indexing). Additionally, their ``controlled" randomness may allow them
to address some potential issues with deterministic sequences (in
particular lattices), e.g., avoiding axis-alignment issues in which
entire rows of samples may become infeasible due to alignment along an
obstacle boundary. In this perspective, achieving deterministic
guarantees is really a matter of i.i.d. sampling versus non-i.i.d. low-dispersion sampling (with deterministic sampling as a prominent case), as opposed to random versus deterministic. Collectively,  our results, complementing and corroborating those in \citep{MSB-SML-KO-ea:01, SML-MB-SRL:04},  strongly suggest that both the study and application of sampling-based algorithms should adopt non-i.i.d. low-dispersion sampling. From a different viewpoint, our results provide a theoretical bridge between sampling-based algorithms with i.i.d.  sampling and non-sampling-based algorithms on regular grids (e.g., D* \citep{AS:95} and related kinodynamic variants \citep{MP-RK-AK:09}).

\emph{Organization:} This paper is structured as follows. In Section
\ref{sec:back} we provide a review of known concepts from
low-dispersion sampling, with a focus on $\ell_2$-dispersion. In
Section \ref{sec:setup} we formally define the optimal path planning
problem. In Section \ref{sec:theory} we present our three main
theoretical results for planning with low-dispersion sequences:
asymptotic optimality, convergence rate, and computational and space
complexity. In Section
\ref{sec:ext} we extend the results from Section~\ref{sec:theory} to
other batch-processing algorithms, non-uniform sampling, and
kinodynamic motion planning. In Section \ref{sec:sims}  we present results from numerical experiments supporting our statements. Finally, in Section \ref{sec:conc}, we draw some conclusions and discuss directions for future work.

\section{Background}\label{sec:back}
A key characteristic of any set of points on a finite domain is its
$\ell_2$-dispersion. This concept will be particularly useful in elucidating
the advantages of deterministic sampling over i.i.d. sampling. As
such, in this section we review some relevant properties and results
on the $\ell_2$-dispersion.

\begin{definition}[$\ell_2$-dispersion]\label{def:L2}
For a finite, nonempty set $S$ of points contained in a
$d$-dimensional compact Euclidean subspace $\mathcal{X}$ with positive
Lebesgue measure, its $\ell_2$-dispersion $D(S)$ is defined as

\begin{equation}\label{eq:L2}
\begin{split}
D(S) &:= \sup_{x\in \mathcal{X}} \min_{s \in S} \|s-x\|_2 = \sup\left\{r> 0 : \exists x\in \mathcal{X} \text{ with }
  B(x,r)\cap S = \emptyset \right\}, \\
\end{split}
\end{equation}
where $B(x,r)$ is the \emph{open} ball of radius $r$ centered at $x$.
\end{definition}

Intuitively, the $\ell_2$-dispersion quantifies how well a space is covered by
a set of points $S$ in terms of the largest open Euclidean ball that touches none of
the points. The quantity $D(S)$ is important in the analysis of path optimality 
as an optimal path may
pass through an empty ball of radius $D(S)$. Hence, $D(S)$ bounds how closely
any path tracing through points in $S$ can possibly approximate that
optimal path.


The $\ell_2$-dispersion of a set of deterministic or random points is
often hard to compute, but luckily it can be bounded by the
more-analytically-tractable $\ell_{\infty}$-dispersion. The
$\ell_{\infty}$-dispersion is defined by simply
replacing the $\ell_2$-norm in
equation~\eqref{eq:L2} by the $\ell_{\infty}$-norm, or max-norm. The
$\ell_{\infty}$-dispersion of a set $S$, which we will denote by
$D_{\infty}(S)$, is related to the $\ell_2$-dispersion in $d$ dimensions
by \citep{HN:92},
\[ D_{\infty}(S) \le D(S) \le \sqrt{d} D_{\infty}(S), \] which allows us to bound $D(S)$ when $D_{\infty}(S)$ is easier to compute. In particular,
an important result due to \citep{PD:83} is that the
$\ell_{\infty}$-dispersion of $n$ independent uniformly sampled points on
$[0,1]^d$ is $O((\log(n)/n)^{1/d})$ with probability 1. Corollary to
this is that the $\ell_2$-dispersion is also $O((\log(n)/n)^{1/d})$
with probability 1.

Remarkably, there are deterministic sequences with $\ell_2$-dispersions of order 
$O(n^{-1/d})$, an improvement by a factor $\log(n)^{1/d}$. (Strictly speaking, one should distinguish point sets, where the number of points is specified in advance, from sequences \citep{SML-MB-SRL:04}---in this paper we will simply refer to both as ``sequences.'') For
instance, the Sukharev sequence \citep{AGS:71}, whereby $[0,1]^d$ is gridded into $n
= k^d$ hypercubes and their centers are taken as the sampled points,
can easily be shown to have $\ell_2$-dispersion of
$(\sqrt{d}/2)\,n^{-1/d}$ for $n = k^d$ points. As we will see in
Section~\ref{sec:theory}, the use of sample sequences with lower
$\ell_2$-dispersions confers on PRM a number of beneficial properties,
thus justifying the use of certain deterministic sequences instead of i.i.d. ones. In the remainder of the paper, we will refer to sequences with $\ell_2$-dispersion
of order  $O(n^{-1/d})$ as \emph{low-dispersion} sequences.  A natural question to ask is whether we can use a
sequence that \emph{minimizes} the $\ell_2$-dispersion. Unfortunately,
such an optimal sequence is only known for $d=2$, in which case it is represented by
the centers of the equilateral triangle tiling \citep{SL:06}. In this
paper, we will focus on the Sukharev \citep{AGS:71} and Halton sequences \citep{JHH:60}, except in two
dimensions when we will consider the triangular lattice as well, though
there are many other deterministic sequences with $\ell_2$-dispersion
of order $O(n^{-1/d})$; see \citep{AY-SML:04, SRL-AY-SML:05,
  AY-SJ-SML-ea:09} for other examples. 

\section{Problem Statement}\label{sec:setup}
The problem formulation follows that in \citep{LJ-ES-AC-ea:15} very closely.
Let $\mathcal X =[0,\, 1]^d$ be the configuration
space, where $d\in \naturals$. Let $\xobs$ be a closed set representing
the obstacles, and let $\xfree = \text{cl}(\mathcal X \setminus
\mathcal \xobs)$ be the obstacle-free space, where $\text{cl}(\cdot)$
denotes the closure of a set. The initial condition is $\xinit\in\xfree$, and the
goal region is $\xgoal\subset\xfree$. A specific path planning problem
is characterized by a triplet $(\xfree, \xinit, \xgoal)$. A function
$\sigma : [0, 1] \to  \reals^d$ is a \emph{path} if it is continuous
and has bounded variation. If $\sigma(\tau)\in \xfree$ for all
$\tau\in [0,\, 1]$, $\sigma$ is said to be
\emph{collision-free}. Finally, if $\sigma$ is collision-free,
$\sigma(0) = \xinit$, and $\sigma(1)\in cl(\xgoal)$, then $\sigma$ is
said to be a \emph{feasible path} for the planning problem $(\xfree,
\xinit, \xgoal)$.

The goal region $\xgoal$ is said to be \emph{regular} if there exists
$ \xi > 0$ such that $\forall y \in \partial \xgoal$, there exists $ z
\in \xgoal$ with $B(z; \xi) \subseteq \xgoal$ and $y \in \partial B(z;
\xi)$ (the notation $\partial \mathcal X$ denotes the boundary of set $\mathcal X$). Intuitively, a regular goal region is a smooth set with a
boundary that has bounded curvature. Regularity is a technical
condition we will use in our results, but is in fact quite weak, as
nearly any goal region can be well-approximated by a regular goal
region. Furthermore, we will say $\xgoal$
is $\xi$-regular if $\xgoal$ is regular for the parameter
$\xi$. Denote the set of all paths by $\Sigma$. A cost function for the planning
problem $(\xfree, \xinit, \xgoal)$ is a function $c:\Sigma \to
\reals_{\geq 0}$; in this paper we will focus on the \emph{arc length}
function. The is then defined as
follows:

\begin{quote}{\bf \normalsize Optimal path planning problem}:
  {\normalsize Given a path planning problem $(\xfree, \xinit, \xgoal)$ with an arc length cost function $c:~\Sigma \to
\reals_{\geq 0}$, find a feasible path $\sigma^{*}$ such that
$c(\sigma^{*} )= \min\{c(\sigma):\sigma \text{ is feasible}\}$. If no
such path exists, report failure.}
\end{quote}

A path planning problem can be arbitrarily difficult if the solution
traces through a narrow corridor, which motivates the standard notion of path
clearance \citep{SK-EF:11}. For a given $\delta>0$, define the
$\delta$-interior of $\xfree$ as the set of all configurations that are at
least a distance $\delta$ from $\xobs$. Then a path is said to have 
strong $\delta$-clearance if it lies entirely inside the
$\delta$-interior of $\xfree$. Further, a path planning problem with optimal
path cost $c^*$ is called $\delta$-robustly feasible if there exists a
strictly positive sequence $\delta_n \rightarrow 0$, and a sequence
$\{\sigma_n\}_{i=1}^n$ of feasible paths such that $\lim_{n
  \rightarrow \infty} c(\sigma_n) = c^*$ and for all $n \in
\mathbb{N}$, $\sigma_n$ has strong $\delta_n$-clearance, $\sigma_n(1)
\in \partial \mathcal{X}_{\text{goal}}$, and $\sigma_n(\tau) \notin
\mathcal{X}_{\text{goal}}$ for all $\tau \in (0,1)$.

Lastly, in this paper we will be considering a generic form of the PRM
algorithm. That is, denote by \PRM (for generic PRM) the algorithm given by
Algorithm~\ref{prmalg}. The function $\texttt{SampleFree}(n)$ is a function that returns a set of $n
\in \mathbb{N}$ points in $\mathcal{X}_{\text{free}}$.  Given a set
of samples $V$, a sample $v\in V$, and a
positive number $r$, $\texttt{Near}(V, v, r)$ is a
function that returns the set of samples 
  $\{u \in V: \|u-v\|_2 < r\}$.  Given two samples  $u, v\in V$, $\texttt{CollisionFree}(u, v)$ denotes the boolean function which is true if and only if the line joining $u$ and $v$ does not intersect an obstacle.
Given a graph $(V, E)$, where the node set  
  $V$ contains $\xinit$ and $E$ is the edge set, 
$\texttt{ShortestPath}(\xinit,V, E)$  is a function returning a shortest path from $\xinit$ to $\xgoal$ in the
graph $(V, E)$ (if one exists, otherwise it reports failure). Deliberately, we do not  specify the definition of
$\texttt{SampleFree}$ and have left $r_n$ unspecified, thus allowing
for any sequence of points---deterministic or random---to be used, with any connection
radius. These ``tuning" choices will be studied in Section \ref{sec:theory}. We want to clarify that we are in no way proposing
a new algorithm, but just defining an umbrella term for the PRM class
of algorithms which includes, for instance, sPRM and \PRMstar as
defined in \citep{SK-EF:11}.
\begin{algorithm}
\caption{gPRM Algorithm}
\label{prmalg}
\algsetup{linenodelimiter=}
\begin{algorithmic}[1]
\STATE $V \leftarrow \{x_{\text{init}}\} \cup \texttt{SampleFree}(n)$; $E \leftarrow \emptyset$
\FORALL{$v \in V$}
\STATE $X_{\text{near}} \leftarrow \texttt{Near}(V \backslash \{v\}, v, r_n)$ 
\FOR{$x \in X_{\text{near}}$} \label{line:forXnear}
\IF{$\texttt{CollisionFree}(v, x)$}  \label{line:collisionfreecheck}
\STATE $E \leftarrow E \cup \{(v, x)\}  \cup \{(x, v)\}$ 
\ENDIF
\ENDFOR
\ENDFOR
\RETURN $\texttt{ShortestPath}(x_{\text{init}}, V, E)$
\end{algorithmic}
\end{algorithm}

\section{Theoretical Results}\label{sec:theory}
In this section we present our main theoretical results.
We begin by proving that \PRM on low-dispersion sequences
is asymptotically optimal, in the deterministic sense, for
connection radius $r_n \in \omega(n^{-1/d})$. Previous work has
required $r_n$ to be at least $\Omega((\log(n)/n)^{1/d})$ for asymptotic optimality.

\begin{theorem}[Asymptotic optimality with deterministic sampling]\label{thrm:AO}
Let $(\mathcal{X}_{\text{free}}, x_{\text{init}},
\mathcal{X}_{\text{goal}})$ be a $\delta$-robustly feasible path
planning problem in $d$ dimensions, with $\delta>0$ and
$\mathcal{X}_{\text{goal}}$ $\xi$-regular. Let $c^*$ denote the arc
length of an optimal path $\sigma^*$,
and let $c_n$ denote the arc length of the path returned by \PRM (or
$\infty$ if \PRM returns failure) with $n$ vertices whose
$\ell_2$-dispersion is $D(V)$ using a radius $r_n$. Then if $D(V) \le
\gamma n^{-1/d}$ for some $\gamma \in \mathbb{R}$ and
\begin{equation}
\label{radiuscondition}
n^{1/d}r_n \rightarrow \infty,
\end{equation}
then  $\lim_{n \rightarrow \infty} \, c_n = c^*$.
\end{theorem}

\begin{proof}
Fix $\varepsilon>0$. By the $\delta$-robust feasibility of the
problem, there exists a $\sigma_{\varepsilon}$ such that
$c(\sigma_{\varepsilon}) \le (1+\varepsilon/3)c^*$ and
$\sigma_{\varepsilon}$ has strong $\delta_{\varepsilon}$-clearance for
some $\delta_{\varepsilon}>0$, see Figure~\ref{fig:rob}.
\begin{figure}[h]
\centering
\subfigure[]{\label{fig:rob} \includegraphics[width=0.3\textwidth]{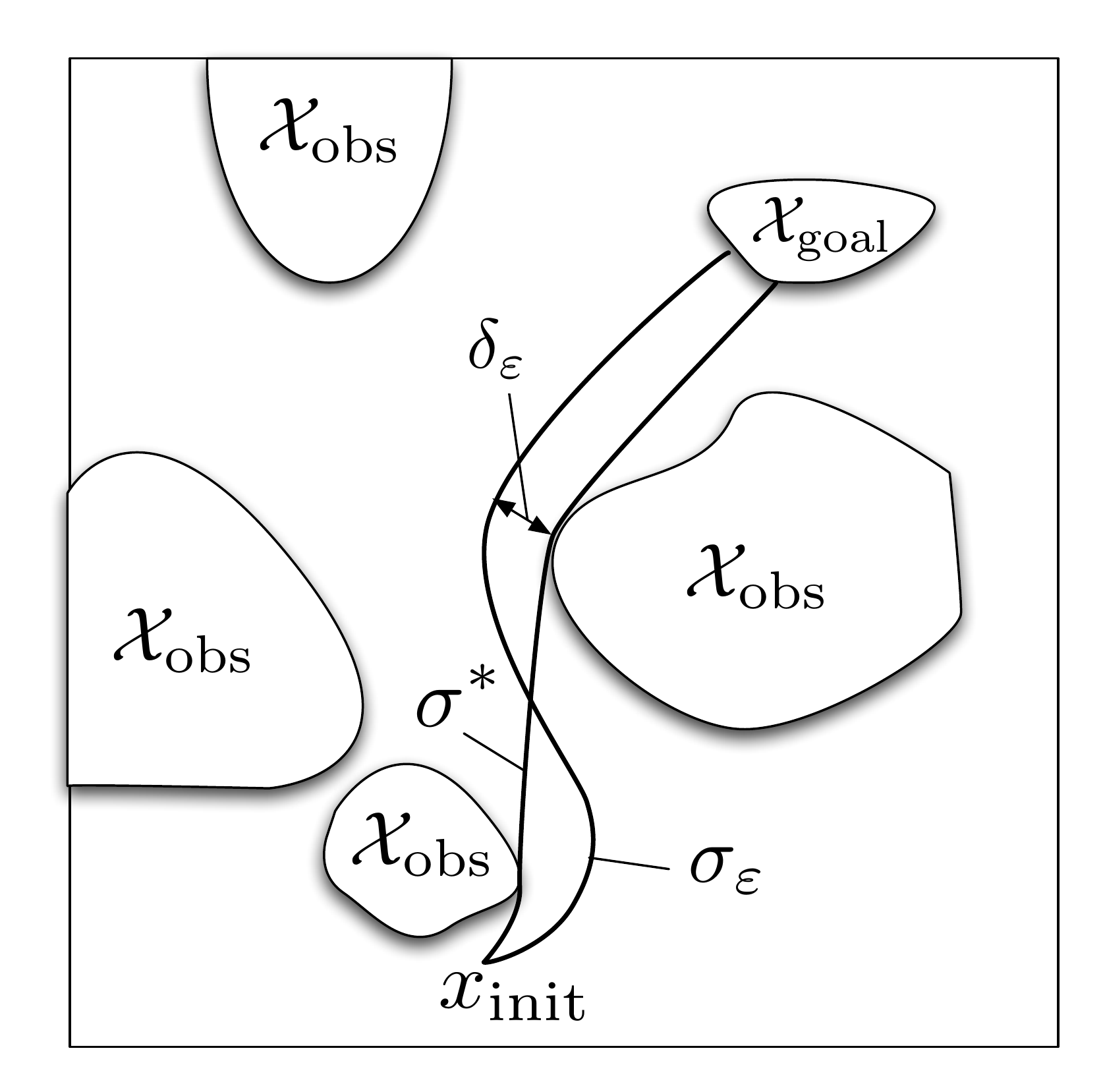}} \, 
\subfigure[]{\label{fig:balls} \includegraphics[width=0.65\textwidth]{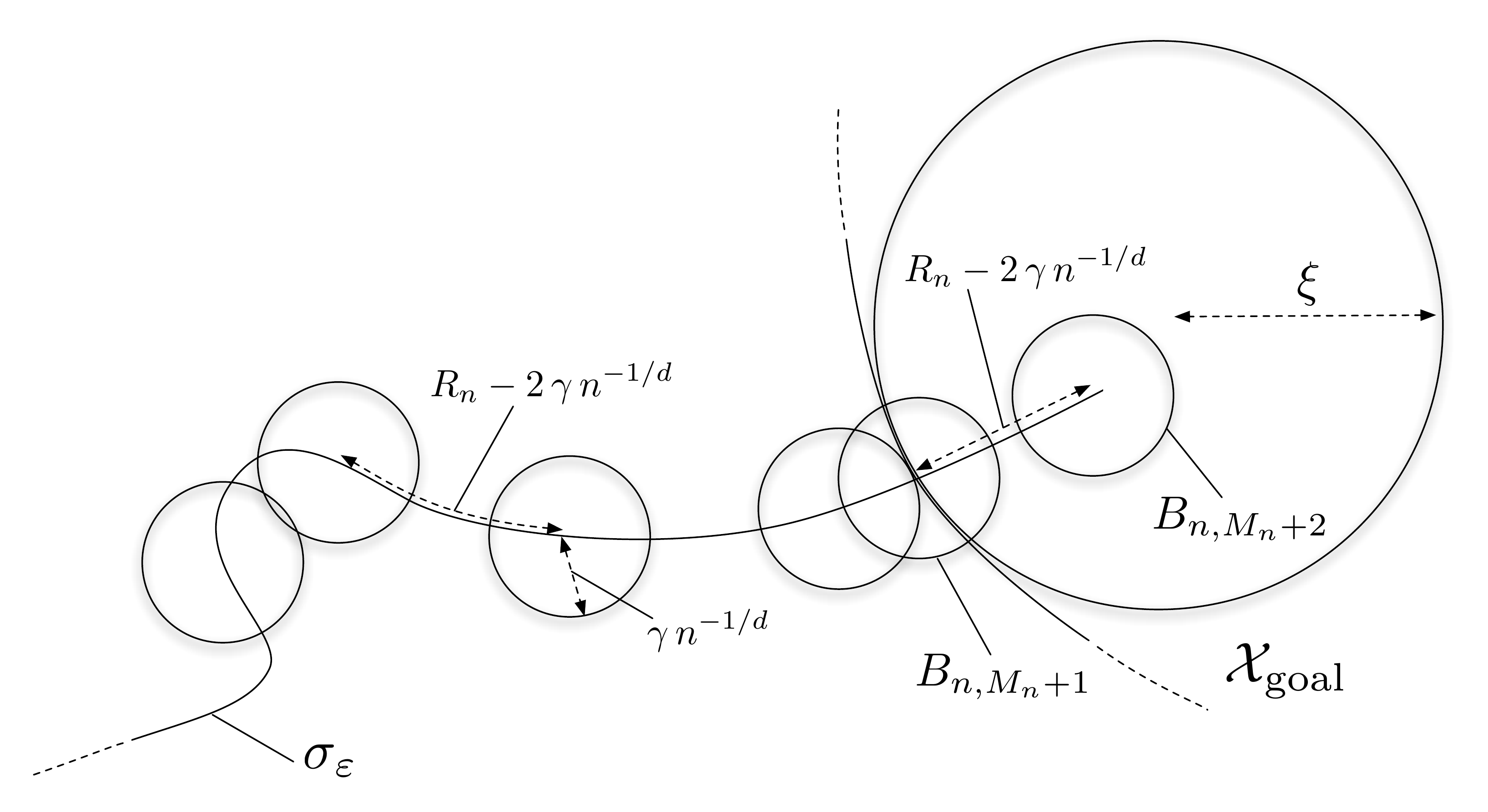} }
\caption{Figure \ref{fig:rob}: Illustration in 2D of $\sigma_{\varepsilon}$ as the shortest strongly
  $\delta_{\varepsilon}$-robust feasible path, as compared to the optimal path
  $\sigma^*$, as used in the proof of Theorem~\ref{thrm:AO}. Figure \ref{fig:balls}: Illustration in 2D of the construction of $B_1,\dots,B_{M_n+2}$ in
  the proof of Theorem~\ref{thrm:AO}.}
\label{fig:rob_balls}
\end{figure}
Let $R_n$ be a sequence such that $R_n
\le r_n$, $n^{1/d}R_n \rightarrow \infty$, and $R_n \rightarrow 0$,
guaranteeing that there exists a $n_0\in \mathbb{N}$ such that for all
$n\ge n_0$,
\begin{equation}
\label{Rn}
(4+6/\varepsilon)\gamma n^{-1/d} \le R_n \le
\min\{\delta_{\varepsilon},\xi,c^*\varepsilon/6\}.
\end{equation}


For any $n\ge n_0$, construct the closed balls $B_{n,m}$ such that
$B_{n,i}$ has radius $\gamma n^{-1/d}$ and has center given by
tracing a distance $(R_n-2\gamma n^{-1/d})i$ from $x_0$ along $\sigma_{\varepsilon}$
(this distance is positive by equation~\eqref{Rn}) until $(R_n-2\gamma
n^{-1/d})i>c(\sigma_{\varepsilon})$. This will
generate $M_n = \lfloor c(\sigma_{\varepsilon})/(R_n-2\gamma n^{-1/d}) \rfloor$
balls. Define $B_{n,M_n+1}$ to also have radius $\gamma n^{-1/d}$ but
center given by the point where $\sigma_{\varepsilon}$ meets
$\mathcal{X}_{\text{goal}}$. Finally, define $B_{n,M_n+2}$ to have
radius $\gamma n^{-1/d}$ and center defined by extending the center of
$B_{n,M_n+1}$ into $\mathcal{X}_{\text{goal}}$ by a distance
$R_n-2\gamma n^{-1/d}$ in the direction perpendicular to $\partial
\mathcal{X}_{\text{goal}}$. Note that by equation~\eqref{Rn},
$B_{n,M_n+2} \subset \mathcal{X}_{\text{goal}}$. See
Figure~\ref{fig:balls} for an illustration.

Since the dispersion matches the radii of all the $B_{n,m}$, each
$B_{n,m}$ has at least one sampled point within it. Label these points
$x_1, \dots, x_{M_n+2}$, with the subscripts matching their respective
balls of containment. For notational convenience, define
$x_0:=x_{\text{init}}$. Note that by construction of the balls, for $i
\in \{0,\dots,M_n+1\}$, each pair of consecutively indexed points
$(x_i,x_{i+1})$ is separated by no more than $R_n \le
r_n$. Furthermore, since $R_n \le \delta_{\varepsilon}$ by equation~\eqref{Rn} above, there cannot be an obstacle between any such pair, and thus
each pair constitutes an edge in the \PRM graph. Thus, we can
upper-bound the cost $c_n$ of the \PRM solution by the sum of the lengths of the edges
$(x_0,x_1),\dots,(x_{M_n+1},x_{M_n+2})$:
\begin{equation*}
\begin{split}
c_n &\le \sum_{i=0}^{M_n+1} \|x_{i+1}-x_i\|  \le (M_n+2)R_n \le \frac{c(\sigma_{\varepsilon})}{R_n-2\gamma n^{-1/d}}R_n + 2R_n \\
&\le c(\sigma_{\varepsilon}) + \frac{2\gamma n^{-1/d}}{R_n-2\gamma n^{-1/d}}c(\sigma_{\varepsilon}) + 2R_n = c(\sigma_{\varepsilon}) + \frac{1}{\frac{R_n}{2\gamma n^{-1/d}}-1}c(\sigma_{\varepsilon})  + 2R_n \\
&\le \left (1+\frac{\varepsilon}{3}\right)c^* + \frac{1}{\frac{3}{\varepsilon}+1}\left (1+\frac{\varepsilon}{3}\right)c^*  + \frac{\varepsilon}{3}c^* = (1+\varepsilon)c^*. \\
\end{split}
\end{equation*}
The second inequality follows from the fact that the distance between
$x_i$ and $x_{i+1}$ is upper-bounded by the distance between the
centers of $B_{n,i}$ and $B_{n,i+1}$ (which is at most $R_n -
2\gamma n^{-1/d}$) plus the sum of their radii (which is $2\gamma
n^{-1/d}$). The last inequality follows from the facts that
$c(\sigma_{\varepsilon})\le (1+\varepsilon/3)c^*$ and equation~\eqref{Rn}.
\end{proof}

Note that if \PRM using $r_n>2D(V)$ reports failure, then there are two possibilities: (i) a solution does not exist, or (ii) all solution paths go through corridors whose widths are smaller than $2 \, D(V)$. Such a result can be quite useful in practice, as solutions going through narrow corridors could be undesirable anyways (see \cite[Section 5]{SML-MB-SRL:04} for the same conclusion).

Next, we relate the solution cost returned by \PRM to the best cost of
a path with strong $\delta$-clearance in
terms of the $\ell_2$-dispersion of the samples used. This is a
generalization of previous convergence rates,
e.g. \citep{LJ-ES-AC-ea:15}, which only apply to obstacle-free
spaces. Previous work also defined
convergence rate as, for a fixed level of suboptimality $\varepsilon$,
the rate (in $n$) that the probability of returning a
greater-than-$\varepsilon$-suboptimal solution goes to zero. In
contrast, we compute the rate (in $n$) that solution suboptimality
approaches zero. Lastly, previous work focused on asymptotic rates in
big-$O$ notation, while here we provide exact upper-bounds for finite
samples.

\begin{theorem}[Convergence rate in terms of dispersion]\label{thrm:convrate}
Consider the simplified problem of finding the shortest feasible path between
two points $x_0$ and $x_f$ in $\xfree$, assuming that both the
initial point and final point have already been sampled. Define
\[\delta_{\text{max}} = \sup\{\delta>0: \exists \text{ a feasible }
\sigma\in\Sigma\text{ with strong $\delta$-clearance}\}\]
and assume $\delta_{\text{max}}>0$. For all
$\delta<\delta_{\text{max}}$, let $c^{(\delta)}$ be the cost of the
shortest path with strong $\delta$-clearance. Let $c_n$ be
the length of the path returned by running \PRM on $n$ points whose
$\ell_2$-dispersion is $D_n$ and using a connection radius
$r_n$. Then for all $n$ such that $r_n>2D_n$ and $r_n<\delta$, 
\begin{equation}\label{eq:rate}
c_n \le \left(1+\frac{2D_n}{r_n-2D_n}\right)c^{(\delta)}.
\end{equation}
\end{theorem}

\begin{proof}
Let $\sigma_{\delta}$ be a feasible path of length $c^{(\delta)}$ with
strong $\delta$-clearance. Construct the balls $B_1,\dots,B_{M_n}$
with centers along $\sigma_{\delta}$ as in the proof of
Theorem~\ref{thrm:AO} (note we are not constructing $B_{M_n+1}$ or
$B_{M_n+2}$), except with radii $D_n$ and centers separated by a
segment of $\sigma_{\delta}$ of arc-length $r_n-2D_n$. Note
that $M_n=\lfloor c^{(\delta)}/(r_n-2D_n)\rfloor$. Then by definition
each $B_i$ contains at least one point $x_i$. Furthermore, each $x_i$ is connected to
$x_{i-1}$ in the \PRM graph (because $x_i$ is contained in the ball of
radius $r_n-D_n$ centered at $x_{i-1}$, and that ball is collision-free),
and $x_f$ is connected to $x_{M_n}$ as well. Thus $c_n$ is upper-bounded by
the path tracing through $x_0,x_1,\dots,x_{M_n},x_f$:
\begin{equation*}
\begin{split}
c_n &\le \|x_1-x_0\| + \sum_{i=2}^{M_n}\|x_i-x_{i-1}\| + \|x_f-x_{M_n}\| \le r_n-D_n + \sum_{i=2}^{M_n}r_n + \|x_f-x_{M_n}\| \\
&\le \left(M_nr_n - D_n\right) + \left(\left(\frac{c^{(\delta)}}{r_n-2D_n}-\left\lfloor
  \frac{c^{(\delta)}}{r_n-2D_n}\right\rfloor\right)(r_n-2D_n) + D_n\right) \\
&= c^{(\delta)} + 2D_nM_n \le \left(1+\frac{2D_n}{r_n-2D_n}\right)c^{(\delta)}, \\
\end{split}
\end{equation*}
where the second and third inequalities follow by considering the farthest
possible distance between neighboring points, given their inclusion in
their respective balls. 
\end{proof}
\begin{remark}[Convergence rate in obstacle-free environments]
Note that when there are no obstacles, $\delta_{\text{max}}=\infty$
and $c^{(\delta)}=\|x_f-x_0\|$ for all $\delta>0$. Therefore, an
immediate corollary of Theorem~\ref{thrm:convrate} is that the
convergence rate in free space of \PRM to the \emph{optimal} solution
is upper-bounded by $2D_n/(r_n-2D_n)$ for $r_n>2D_n$.
\end{remark}

\begin{remark}[Practical use of convergence rate]
Theorem \ref{thrm:convrate} provides a convergence rate result to a
shortest path with strong $\delta$-clearance. This result is useful
for two main reasons. First, in practice, the objective of path
planning is often to find a high-quality path with some ``buffer distance"
from the obstacles, which is precisely captured by the notion of
$\delta$-clearance.  Second, the convergence rate in equation
\eqref{eq:rate} could be used to certify the performance of gPRM (and
related batch planners) by placing some assumptions on the obstacle
set (e.g., minimum separation distance among obstacles and/or
curvature of their boundaries). For instance, consider a
deterministic low-dispersion sequence with dispersion upper bounded
by $\gamma\, n^{-1/d}$, and assume one has time to plan on $\bar n$
samples. Choose $r_{\bar n} = 2\psi \gamma \, \bar n^{-1/d}$ for some
$\psi>1$ and choose $\delta_0 >r_{\bar n}$ (one can
readily verify that such a choice of $r_{\bar n} $ satisfies the
assumptions of Theorem \ref{thrm:convrate}). Then one can state the
deterministic guarantee ``For \emph{all} planning problems where there
exists a feasible $\delta_0$-clear path (see Figure
\ref{fig:convrate}), the cost of the path returned is within a factor $1/(\psi -1)$ of the cost of the \emph{best} $\delta_0$-clear path." For given values of $\psi$ and $\delta_0$, one can therefore ``certify'' the performance of the planner for a desired target performance.
%
%
%
%
An interesting avenue for future
research is to use information about the curvature of the obstacles to
quantify the difference between $c^{(\delta)}$ and the true optimal
cost $c^*$.
\end{remark}

\begin{figure}[h!]
\center
\includegraphics[width=.3\textwidth]{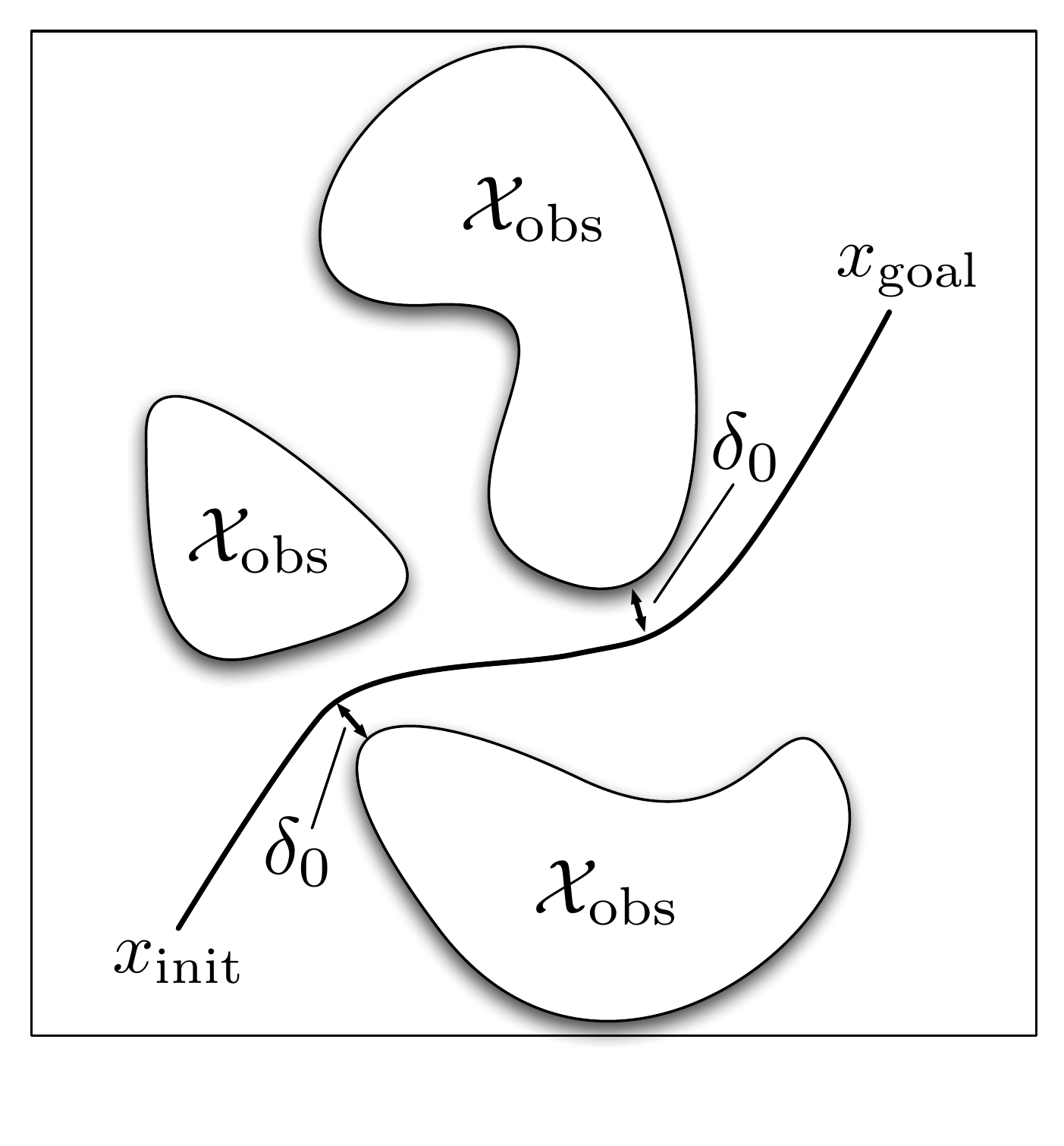} 
\caption{An example of a planning problem with a feasible $\delta_0$-clear path. For a given clearance parameter $\delta_0$ and suboptimality factor $\psi$, one can readily determine the number of samples (and, hence, computation time) that is \emph{deterministically} guaranteed to return a path whose cost is within $1/(\psi -1 )$ of the cost of the \emph{best} $\delta_0$-clear path, for \emph{all} planning problems where a feasible $\delta_0$-clear path exists.
} 
\label{fig:convrate}
\end{figure}

Both the asymptotic optimality and convergence rate results can be
extended to other batch planners such as Lazy-PRM or \FMT\!, as
discussed in Section~\ref{sec:othalg}. 

Lastly, we show that using a low-$\ell_2$-dispersion lattice sample
set, an asymptotically-optimal (AO) version of \PRM can be run that has lower-order
computational complexity than any existing AO algorithm, namely
$\omega(n)$ instead of $O(n\log(n))$.

\begin{theorem}[Computational complexity with deterministic
  sampling]\label{thrm:complexity}
  \PRM run on $n$ samples arranged in a $d$-dimensional cubic lattice
  with connection radius $r_n$ satisfying
  equation~\ref{radiuscondition} has computational complexity
\begin{equation}
\label{complexity}
O(n^2r_n^d).
\end{equation}
Furthermore, as long as $nr_n^d \nrightarrow 0$, the computational
complexity is also $O(n^2r_n^d)$.
\end{theorem}
\begin{proof}
The algorithm \PRM has three steps: (1) For each sampled point $x$, it needs to compute
which other sampled points are within a distance $r_n$ of $x$. (2) For
each pair of sampled points within $r_n$ of one another, their
connecting edge needs to be checked for collision, and if
collision-free, its edge-length needs to be computed. (3) The shortest
path through the graph produced in steps (1) and (2) from the initial
point to the goal region needs to be computed.

The lattice structure makes it trivially easy to bound a point's
$r_n$-neighborhood by a bounding hypercube with side-length $2r_n$,
ensuring only $O(nr_n^d)$ nearby points need to be checked for each of
the $n$ samples, so this step takes $O(n^2r_n^d)$ time.

In step (2), one collision-check and at most one cost computation needs to be
performed for each pair of points found in step (1) to be within $r_n$
of one another. The number of such pairs can be bounded above by the
number of sampled points times the size of each one's neighborhood,
leading to a bound of the form $O(n\cdot nr_n^d)$. Thus step (2) takes
$O(n^2r_n^d)$ time.

After steps (1) and (2), a weighted (weights correspond to edge
lengths) graph has been constructed on $n$ vertices with a number of
edges asymptotically upper-bounded by $n^2r_n^d$. One more property of
this graph, because it is on the cubic lattice, is that the number of
distinct edge lengths is asymptotically upper-bounded by $nr_n^d$. An
implementation of Dijkstra's algorithm for the single source shortest
path problem is presented in \citep{JBO-KM-KS:10} with running time
linear in both the number of edges and the number of vertices times
the number of distinct edge lengths. Since both are $O(n^2r_n^d)$,
that is the time complexity of step (3). 

The space complexity is proportional to the number of edges plus the
number of vertices, which are $O(n^2r_n^d)$ and $O(n)$,
respectively. By assumption that $nr_n\nrightarrow 0$, $O(n^2r_n^d)$
will be the (possibly co-) dominant term.
\end{proof}

Since Theorem~\ref{thrm:AO} allows $r_n\in \omega(n^{-1/d})$ while
maintaining AO, Theorem~\ref{thrm:complexity} implies that
cubic-lattice sampling allows for an AO algorithm with computational
and space complexity $\omega(n)$. All other AO algorithms in the
literature have computational and space complexity at least
$O(n\log(n))$. While the use of an $r_n\in \omega(n^{-1/d})$ makes the graph
construction phase (steps (1) and (2)) $\omega(n)$, step (3) would in
general take longer, as shortest-path algorithms on a general graph
with $n$ vertices requires $\Omega(n\log(n))$. Thus the lattice structure must be
leveraged to improve the complexity of step (3)---we discuss the
limitations implied by this in the next section.

\section{Extensions}\label{sec:ext}
In this section we discuss several extensions to the results presented  in Section \ref{sec:theory}.  In particular, we discuss
extensions  to alternative implementations and other types of batch-processing algorithms (Sections \ref{sec:othalg} and \ref{sec:complex}), to
non-uniform sampling sequences (Section \ref{sec:nus}), and to kinodynamic motion planning (Section \ref{sec:dkp}).

\subsection{Convergence Results for Other Batch-Processing Algorithms}
\label{sec:othalg}

The theoretical convergence results of the previous section
(Theorems~\ref{thrm:AO} and \ref{thrm:convrate}) extend to
alternative implementations of gPRM and other batch-processing algorithms.  We first discuss a variant of gPRM which makes connections based on $k$-nearest-neighbors instead of a fixed connection radius. This variant, referred to as $k$-nearest
gPRM, has the advantage of being more adaptive to different obstacle
spaces than its connection-radius counterpart (we refer the reader to
\cite[Section 5.3]{LJ-ES-AC-ea:15} for  a more detailed discussion
about the relative benefits of a $k$-nearest-neighbors
implementation). Assuming that $k$-nearest gPRM is implemented so that
the number of neighbors $k_n$ (a function of $n$, as $r_n$) is
taken for each sample to be roughly connected to its
$r_n$-neighborhood if there were no obstacles, then deterministic
asymptotic optimality and rate convergence guarantees can be readily
derived. Specifically, assuming $k_n = (1+\epsilon)n\zeta_dr_n^d$,
where $\epsilon>0$ and $\zeta_d$ is the volume of the unit ball in $d$
dimensions, then a graph constructed in $k$-nearest gPRM with $k_n$ neighbors will
asymptotically contain all the edges (and more) of the graph
constructed in gPRM with radius $r_n$, and thus will return a path at
least as low-cost.

A second example is Lazy-PRM \citep{RB-LK:00}, where
the extension of the theoretical convergence results  is straightforward, as the path returned by Lazy-PRM is
identical to that returned by gPRM (using the same radius). 

A third example is \FMT\! \citep{LJ-ES-AC-ea:15}. The only difference in the proof for
\FMT from that for gPRM is that $c_n$ cannot na\"ively be upper-bounded
by the length of a path tracing through a sequence of points in the
gPRM graph, as \FMT may not use some edges. However, as shown in the
proof of Lemma 4.2 (pages 912--913) in \cite[Appendix A]{LJ-ES-AC-ea:15},
\FMT\!'s $c_n$ \emph{can} be bounded by the length of a path tracing
through a sequence of points in the gPRM graph \emph{if those points
  are contained in a suitable sequence of balls}. In particular, the
sequence of balls needs to be sufficiently far from obstacles and
adjacent balls need to be sufficiently close together. It is easy
to check that the sequences $\{B_i\}$ used in the proofs of the
previous section satisfy these conditions.

\subsection{Complexity Results for Different Sampling Schemes and Other Batch-Processing Algorithms}
\label{sec:complex}

The complexity result in Theorem~\ref{thrm:complexity}, strictly speaking, only applies to  \PRM as defined in Algorithm \ref{prmalg} run on $n$ samples arranged in a $d$-dimensional cubic lattice. Using other low-dispersion but non-lattice sampling schemes, such as the
Halton/Hammersley sequence, would not give $O(nr_n^d)$ distinct edge
lengths, which precludes the use of the improved implementation
of Dijkstra's algorithm \citep{JBO-KM-KS:10}. A $k$-nearest-neighbor implementation
(where $k_n$ is selected as specified in Section \ref{sec:othalg} in order to ensure convergence), even on a lattice, 
would also no longer in general have $O(nr_n^d)$ distinct edge
lengths. Lastly, other 
batch-processing algorithms, such as \FMT\!, do not separate graph
construction and shortest-path computation, again precluding, at least
na\"ively, the use of the improved implementation
of Dijkstra's algorithm. 

In all the above cases, all the same computational advantages stated
in the proof of Theorem~\ref{thrm:complexity} would hold,
\emph{except} the advantage from the sped-up
shortest-path algorithm in step (3). In practice the shortest-path
algorithm is typically a trivial fraction of path planning runtime, as
it requires none of the much-more-complex edge-cost and collision-free
computations. In other words, the constant hidden in the big-O
notation for the shortest-path algorithm is drastically smaller than
that in the other steps. Thus, a practical approximation of the
runtime could ignore the shortest-path component, in which case the
result of Theorem~\ref{thrm:complexity} can be extended to all of the
aforementioned examples.

\subsection{Non-uniform Sampling}\label{sec:nus}
A popular method for incorporating prior knowledge about a problem is
to draw samples in a way that is not uniform throughout the
configuration space. Increased sampling density in areas that are
especially hard to traverse gives a motion planning algorithm help in
that area. Thus if these areas can be identified a priori or even
adaptively (see, for example, \citep{VB-MHO-vdS:99}), observed convergence can
be substantially sped up. Sampling non-uniformly in the i.i.d. setting
is often fairly straightforward, for instance by simply adding
sample-rejection rules to the uniform strategy \citep{DH-JCL-HK:06}. As long as there
remains a baseline density of samples everywhere in the space, similar
worst-case (i.e., even if the non-uniformity of sampling was poorly chosen
for the problem at hand) guarantees to the uniform case can still be
made (see, for example, \citep{LJ-ES-AC-ea:15}). The notion of density
is less clear in the deterministic case, but we can fall back on the notion of 
 $\ell_2$-dispersion instead. Indeed, the theoretical convergence results in
this paper are not specific to any deterministic sampling strategy,
but simply make a requirement on the $\ell_2$-dispersion. Thus, many
non-i.i.d., \emph{non-uniform} sampling strategies fit into the analysis
given here. Consider, for instance, using a sequence with
$\ell_2$-dispersion upper-bounded by $\gamma n^{-1/d}$ to sample $n/2$
points, and then sampling $n/2$ more points in any way (e.g., biased
towards sampling near obstacles). The resulting set of $n$ points will
have $\ell_2$-dispersion upper-bounded by $2^{1/d}\gamma n^{-1/d}$ and thus
still be a low-$\ell_2$-dispersion sequence. There is a rich
literature in different methods for non-uniform sampling, and an
interesting future direction will be to investigate how best to adapt
such methods to the deterministic context.

\subsection{Deterministic Kinodynamic Planning}\label{sec:dkp}
Another important extension is to motion planning with differential
constraints. In particular, we consider here the extension to systems with linear
affine dynamics of the form: $\dot{\mathbf{x}}(t) = A\mathbf{x}(t) + B\mathbf{u}(t) + \mathbf{c}$, where $A$, $B$, and $\mathbf{c}$ are constants. The extension of  the
$\ell_2$-dispersion-based analysis of this paper to that case poses some challenges. The key
roadblock is that the $\ell_2$-dispersion is no longer a particularly
accurate measure of how suitable a set of points is to track an optimal
differentially-constrained path. Essentially, Euclidean balls must
be replaced by ``perturbation'' balls \citep{ES-LJ-MP:15b}, which are
high-dimensional ellipses. To be clear, by a high-dimensional
ellipse we mean a volume defined by
\begin{equation}
\label{ellipse}
\{x: x^TQx < r\}
\end{equation}
for some positive-definite matrix $Q$ and scalar $r$. Although such
ellipses may be inner-bounded
by a Euclidean ball, this (poor) approximation adds an exponential
factor of the \emph{controllability index} of the pair $(A,B)$ \citep{ES-LJ-MP:15b}
to the analysis. (Assuming the pair $(A,B)$ is controllable, the controllability indices $\nu_i$ give a fundamental notion of how difficult a linear system is to control in the various directions, see \citep[p. 431]{TK:80} or \cite[pp. 150]{CTC:95} for a detailed treatment. The number of controllability indices is equal to the number of control inputs, that is to the number of columns of $B$. The maximum, that is  $\nu = \max \nu_i$, is referred to as the controllability index of the pair $(A,B)$.)
 The following theorem (whose proof is largely based on the analysis framework devised in \citep{ES-LJ-MP:15b}) summarizes the optimality  result. Here gDPRM is just Algorithm 1 except that
\texttt{Near} uses the cost in \cite[equation (2)]{ES-LJ-MP:15b}
instead of arc-length.
\begin{theorem}[Asymptotic optimality with deterministic sampling for systems with linear affine dynamics]\label{thrm:gDPRM}
Under the assumptions of \cite[Theorem VI.1]{ES-LJ-MP:15b}, gDPRM with
deterministic low-dispersion sampling is asymptotically-optimal for
\begin{equation}
\label{dprm-r}
r_n = C_1 n^{-1/(\nu d)}
\end{equation}
for some constant $C_1$, while gDPRM with i.i.d. uniform sampling is asymptotically optimal for
\begin{equation}
\label{dprm-ed}
r_n = C_2 \left(\frac{\log(n)}{n}\right)^{1/\tilde{D}}
\end{equation}
for some constant $C_2$, where $\tilde{D} = (d+\sum \nu_i^2)/2$,  the $\nu_i$ are the
controllability indices of the pair $(A,B)$, and $\nu = \max \nu_i$.
\end{theorem}
\begin{proof}[Proof sketch] For the sake of
  brevity, we only describe here the changes needed to the theory in
  \citep{ES-LJ-MP:15b}, as the results have much in common.
The proof of \eqref{dprm-r} is nearly identical to that of
\cite[Theorem VI.1]{ES-LJ-MP:15b} except with a low-dispersion
analogue of \cite[Theorem IV.6]{ES-LJ-MP:15b} which uses the same
$r_n$ rate as in \eqref{dprm-r}. To see this analogue result, note
that \cite[Lemma IV.5]{ES-LJ-MP:15b} holds deterministically under
low-dispersion sampling with the $\log$ term removed from the
condition on the volume $\mu[T_k]$. Then the proof of the analogue of
\cite[Theorem IV.6]{ES-LJ-MP:15b} only requires construction
of the sets $S_k$ (no $T_k$), again without the $\log$ term. From the
deterministic version of \cite[Lemma IV.5]{ES-LJ-MP:15b}, we have that
for the scaling of $r_n$ as in \eqref{dprm-r}, \emph{every} set $S_k$ will
contain a sample, and equation \eqref{dprm-r} follows.

Finally, equation \eqref{dprm-ed} is a direct corollary of \cite[Theorem
VI.1]{ES-LJ-MP:15b}.
\end{proof}

If $\nu=1$ (i.e., all directions are ``equally difficult'' to control),
deterministic sampling and our analysis show all the same benefits as
in the case of the path planning (non-kinodynamic) problem by getting
rid of the $\log(n)$ term required by i.i.d. sampling without changing the exponent (as in this case, $\nu d = d$ and $\tilde{D} = d$). Note that a special case where $\nu=1$ is represented by the single-integrator model $\dot{\mathbf{x}}(t) = \mathbf{u}(t)$, which effectively reduces the kinodynamic planning problem to the path planning problem stated in Section \ref{sec:setup}---in this sense, Theorem \ref{dprm-r} recovers Theorem \ref{thrm:AO} when $\dot{\mathbf{x}}(t) = \mathbf{u}(t)$. However, in general,
the exponent for the case of deterministic low-dispersion sampling
(i.e., the exponent in equation~\eqref{dprm-r}) may be worse. For instance, for the double-integrator model in three dimensions, namely $\ddot{\mathbf{x}}(t) = \mathbf{u}(t)$ and $d=6$, the three controllability indices are $\nu_1=\nu_2=\nu_3=2$. As a consequence, one obtains $\nu d = 12$ and  $\tilde{D} = 9$, and the radius in equation \eqref{dprm-r} (i.e., for the deterministic case) is \emph{larger} that the radius for equation \eqref{dprm-ed} (i.e., for the case with i.i.d. uniform sampling). 

This is \emph{not} to say that 
deterministic sampling is necessarly inappropriate or not advantageous for differentially-constrained
problems, but just that the analysis used here is inadequate (most critically, we crudely inner-bound ellipses via 
Euclidean balls). Our analysis
does, however, suggest possible ways forward. One could consider a measure of
dispersion which applies more specifically to ellipses, and possibly
tailor a deterministic sequence to be low-dispersion in this sense. To our
knowledge, no assessment of sample sequences in terms of this type of
dispersion has been performed previously, and this represents a theoretically and practically important direction for future research (together with studying tailored notions of sampling sequences and dispersion for other classes of dynamical systems, e.g., driftless systems).

We will further study kinodynamic motion planning via deterministic sampling sequences through numerical experiments in Section \ref{sec:sims}.

\section{Numerical Experiments}\label{sec:sims}

In this section we numerically investigate the benefits of planning
with deterministic low-dispersion sampling instead of i.i.d. sampling. Section
\ref{sec:env} overviews the simulation environment used for this
investigation. Section \ref{sec:methods} details the deterministic
low-dispersion sequences used, namely, lattices and the Halton
sequence. Several simulations are then introduced and results compared
to i.i.d. sampling in Sections \ref{sec:tests} and
\ref{sec:sum}. Finally, we briefly discuss non-i.i.d., random, 
low-dispersion sampling schemes in Section \ref{sec:nondet}.

\subsection{Simulation Environment}\label{sec:env}

Simulations were written in C++, MATLAB, and Julia  \citep{JB-SK-VBS-AE:12}, and run using a Unix operating system with a 2.3 GHz processor and 8 GB of RAM. The C++ simulations were run through
the Open Motion Planning Library (OMPL) \citep{IAS-MM-LEK:12}. The
planning problems simulated in OMPL were rigid body problems,
the simulations in MATLAB involved point robots and kinematic chains,
and those in Julia incorporated kinodynamic constraints. For each planning
problem, the entire implementation of \PRM was held fixed (including
the sequence of $r_n$) except for the sampling strategy. Specifically,
for all simulations (except the chain and kinodynamic simulations), we use as connection radius $r_n  = \gprm \, (\log(n)/n)^{1/d}$, where $\gprm  = 2.2\left (1+1/d)^{1/d}(1/\zeta_d \right )^{1/d}$
and $\zeta_d$ is again the volume of the unit ball in $d$-dimensional
Euclidean space. This choice ensures asymptotic optimality both for
i.i.d. sample sequences \citep{SK-EF:11, LJ-ES-AC-ea:15} and deterministic low-dispersion sequences (Theorem \ref{thrm:AO}). Because this is an \emph{exact}
``apples-to-apples'' comparison, we do not present runtime results but
only results in terms of sample count, which have the advantage that
they do not depend on the specific hardware or software we use. (Note
that drawing samples represents a trivial fraction of the total algorithm
runtime. Furthermore, as mentioned in the introduction, deterministic
sampling even allows for possible speedups in computation.)
The code for these results can be found at
\url{https://github.com/stanfordasl}.

\subsection{Sampling Sequences}\label{sec:methods}

We consider two deterministic low-dispersion sequences, namely the Halton sequence \citep{JHH:60} and lattices. Halton sampling is based on a generalization of
the van der Corput sequence and uses prime numbers as its base to form
a deterministic low-dispersion sequence  of points \citep{JHH:60,
  SML-MB-SRL:04}. Lattices in this work were implemented as a triangular lattice in two
dimensions and a Sukharev grid in higher dimensions
\citep{AGS:71}. 

Along with the benefits described
throughout the paper, lattices also present some challenges
\citep{SML-MB-SRL:04}. First, for basic
cubic lattices with the same number of lattice points $k$ per side,
the total number of samples $n$ is constrained to $k^d$, which limits the
potential number of samples. For example, in 10 dimensions, the first four available sample counts are 1, 1024,
59,049, and 1,048,576. There are some strategies to allow for
incremental sampling \citep{AY-SML:04, SRL-AY-SML:05, AY-SJ-SML-ea:09},
but in this paper we overcome this difficulty by simply ``weighting"
dimensions. Explicitly, we allow each side to have a different
number of lattice points. Independently incrementing each side's
number of points by 1 increases the allowed resolution of $n$ by a
factor of $d$, as it allows all numbers of the form
$n=(k-1)^{m}(k)^{d-m}$. 

Second, lattices are sensitive to axis-alignment effects, whereby an axis-aligned obstacle can invalidate an entire axis-aligned row of samples. 
A simple
solution to this problem is to rotate the entire lattice by a fixed
amount (note this changes nothing about the $\ell_2$-dispersion or
nearest-neighbor sets). We chose to rotate each dimension by $10\pi$
degrees as an arbitrary angle far from zero (in practice, problems
often show a ``preferential'' direction, e.g., vertical, so there may
be an a priori natural choice of rotation angle). 

Finally, we include examples in SE(2) and SE(3) because they are
standard, even though the cost function is no longer arc-length in the
configuration space, but closer to arc-length in the
translational dimensions. As this is technically outside the theory in this paper, we
must consider an appropriate analogue of deterministic low-dispersion sampling. In
the na\"ive Sukharev lattice, each translational lattice point
represents many rotational orientations. For instance, in SE(3), this
means there are only $k^3$ translational points when $n=k^6$,
providing a poor extension of the theory in this paper to the rigid
body planning problem. A better approximate extension is to ``spread''
the points, which entails de-collapsing all the rotational orientations at each
translational lattice point by spreading them deterministically in the
translational dimensions around the original point. An example of this process with one rotational and two spatial dimensions is shown in Figure \ref{fig:spreading}.  A similar phenomenon occurs with kinodynamic sampling, in which case it is beneficial to spread velocity samples. For the double-integrator model, as shown in Figure \ref{fig:vspreading}, this process entails offsetting velocity samples from their original positions by an amount proportional to their velocities. To then reduce unfavorable structure between neighboring lattice points and increase variation in connection types, alternating samples were rotated 180 degrees for the final implementation in Figure \ref{fig:DI_lattice_flip}.

\begin{figure}[h!]
\center
	\subfigure[]{\label{fig:unspread} \includegraphics[width=.3\textwidth]{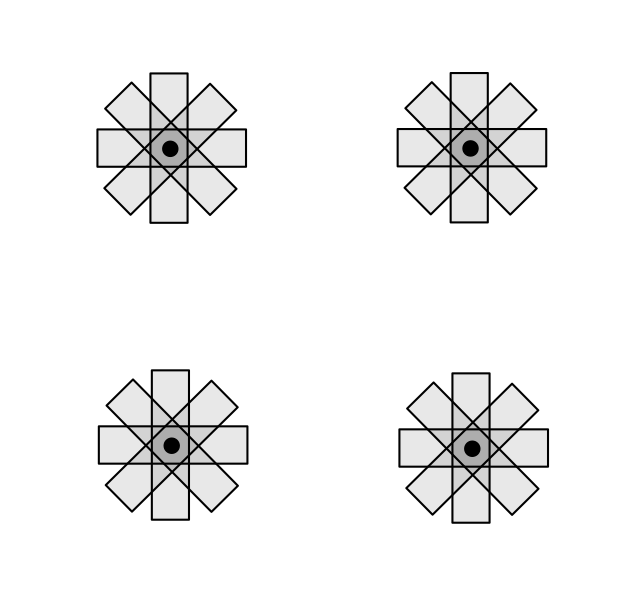}} \qquad\qquad
    \subfigure[]{\label{fig:spread} \includegraphics[width=.3\textwidth]{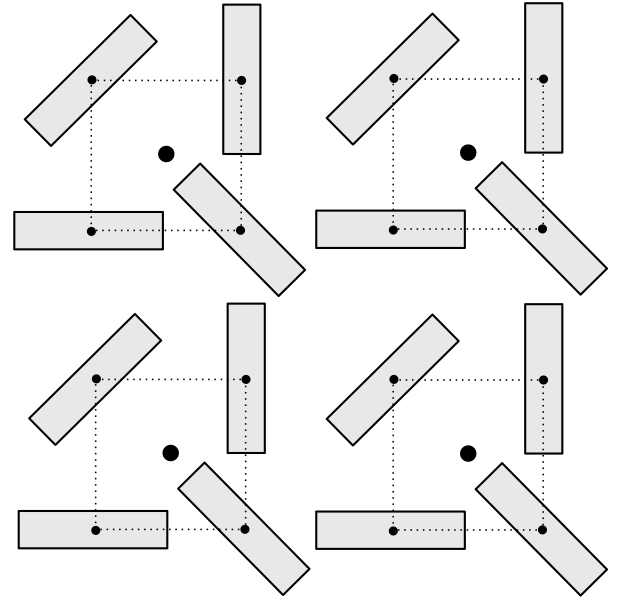}} 
\caption{Figure \ref{fig:unspread}: Rectangular rigid body lattice samples with one rotational and two spatial degrees of freedom. Figure \ref{fig:spread}: A ``spread'' lattice formed by spreading the different rigid body orientations spatially around the original point results in a better spatial coverage of the configuration space.} 
\label{fig:spreading}
\end{figure}

\begin{figure}[h!]
\center
	\subfigure[]{\label{fig:DI_lattice_unspread} \includegraphics[width=0.3\textwidth]{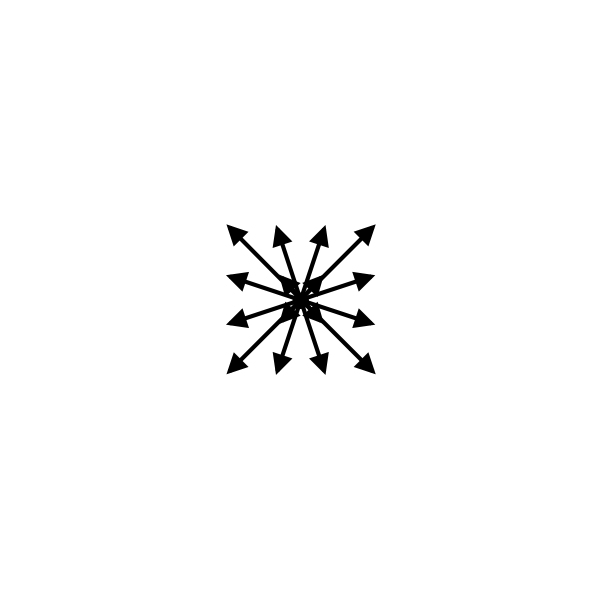}}
	\subfigure[]{\label{fig:DI_lattice} \includegraphics[width=0.3\textwidth]{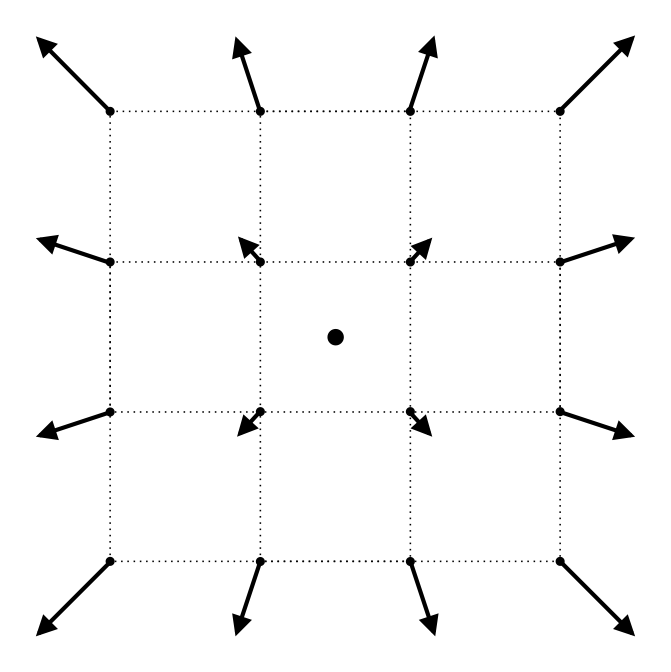}}
	\subfigure[]{\label{fig:DI_lattice_flip} \includegraphics[width=.3\textwidth]{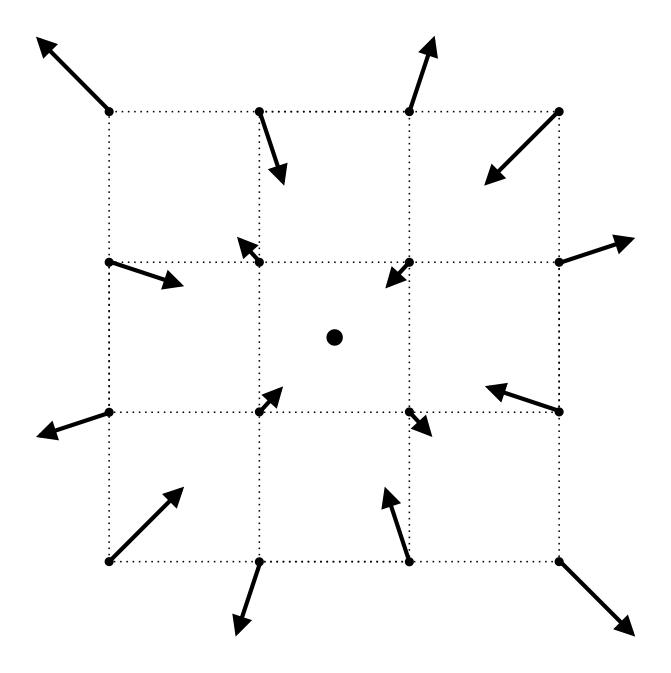}} 
\caption{Figure \ref{fig:DI_lattice_unspread}: A lattice for the double-integrator model with two spatial and two velocity degrees of freedom. Figure \ref{fig:DI_lattice}: A ``spread'' lattice formed by offsetting the velocity samples from their original positions (by an amount proportional to their velocities) to obtain better spatial coverage. Figure \ref{fig:DI_lattice_flip}: The same lattice with alternating samples rotated 180 degrees to increase variation in types of connections and reduce unfavorable structure (this implementation was used for our results).}
\label{fig:vspreading}
\end{figure}


\subsection{Simulation Test Cases}\label{sec:tests}

Within the MATLAB environment, results were simulated for a point
robot within an Euclidean unit hypercube with a variety of geometric obstacles over a range of dimensions. First, rectangular obstacles in 2D and 3D were generated with a fixed configuration that allowed for several homotopy classes of solutions. A 2D
maze with rectangular obstacles was also created (Figure
\ref{fig:maze}). These sets of rectangular obstacles are of particular
interest as they represent a possible ``worst-case'' for lattice-based
sampling because of the potential for axis alignment between samples
and obstacles. The results, shown for
the 2D maze in Figure \ref{fig:rectObs} and for all experiments in
Table \ref{table:summary}, show that Halton and lattice sampling
outperform  i.i.d. sampling in both success rate and solution
cost.


\begin{figure}[h!]
\centering
	\subfigure[]{\label{fig:maze} \includegraphics[width=.3\textwidth]{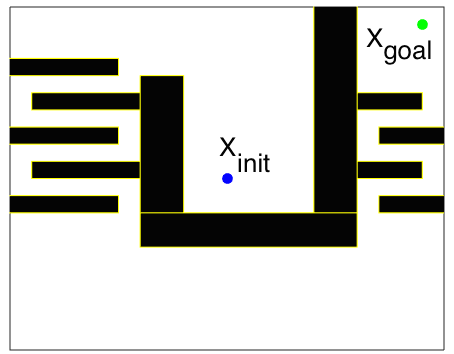}} 
    \subfigure[]{\label{fig:maze_cost} \includegraphics[width=.3\textwidth]{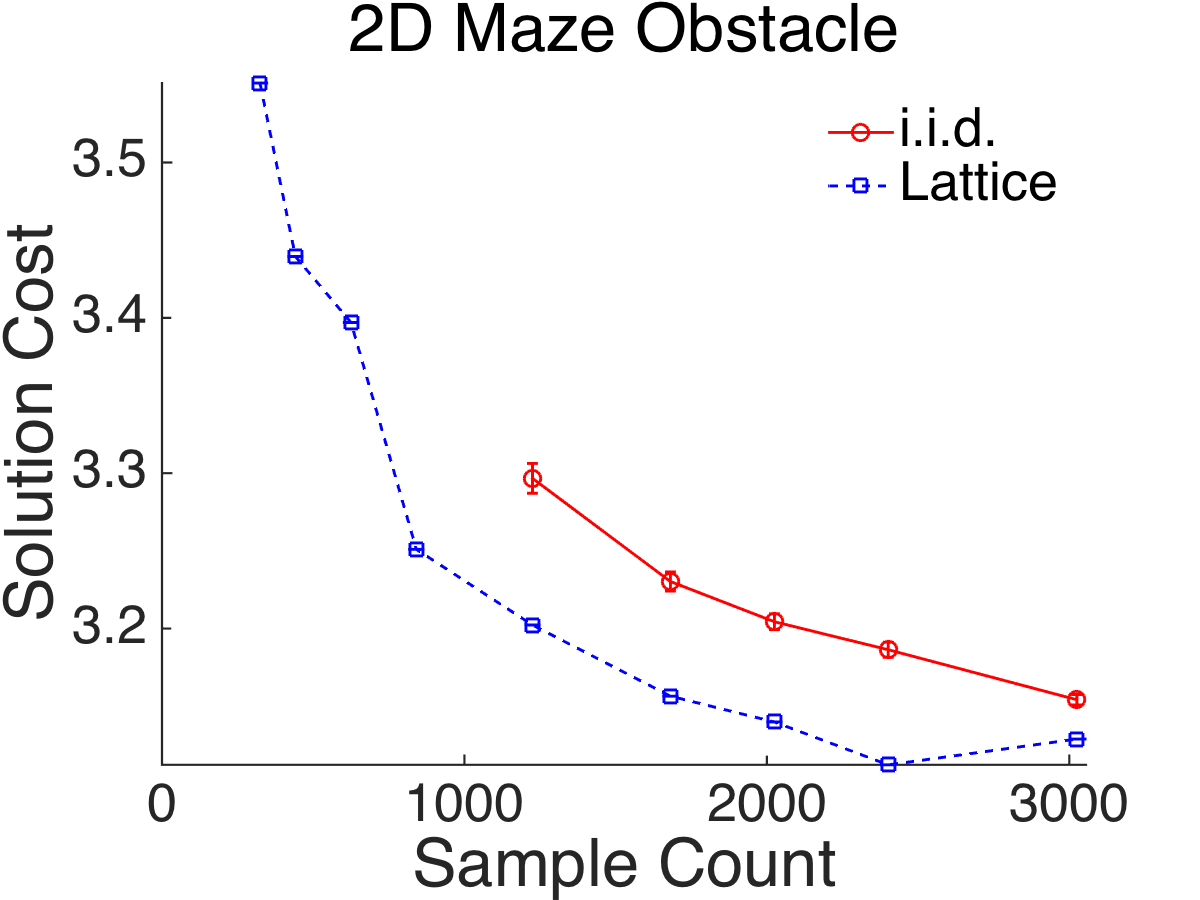}}
	\subfigure[]{\label{fig:maze_success} \includegraphics[width=.3\textwidth]{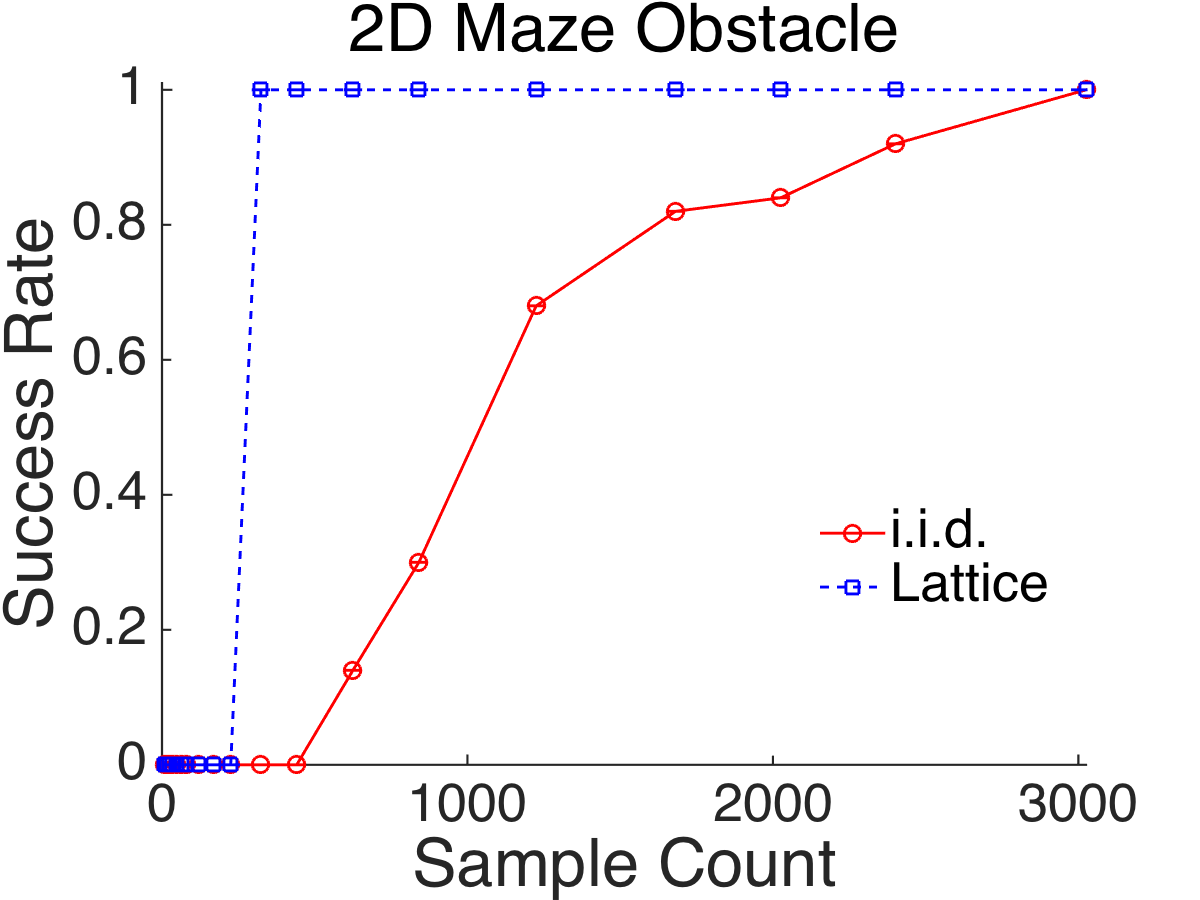}}
\caption{Figure \ref{fig:maze}: The planning setup for a point robot with rectangular obstacles in a 2D maze. Figures \ref{fig:maze_cost} and \ref{fig:maze_success}:  The results for solution cost and success rate versus sample count (averaged over 50 runs). For clarity we only report data for i.i.d. sampling and lattice sequences, results including Halton sampling are reported in Table \ref{table:summary}. Only points with greater than 50\% success are shown in Figure  \ref{fig:maze_cost}.}
\label{fig:rectObs}
\end{figure}


To illustrate planning with rectangular obstacles in higher dimensional
space, we constructed a recursive maze obstacle environment. Each
instance of the maze consists of two copies of the previous dimension's maze, separated by an obstacle with an opening through the new dimension, as detailed in \citep{LJ-ES-AC-ea:15}. Figure \ref{fig:recMaze2} shows the maze in 2D and Figure \ref{fig:recMaze3} shows the maze in 3D with the two copies of the 2D maze in black and the opening in red. Halton and lattice sampling conferred similar benefits in the recursive mazes
in 2D, 3D, 4D, 5D, 6D, and 8D as they did in other simulations (see Table
\ref{table:summary}).

Along with the rectangular obstacles, hyperspherical
obstacles within a Euclidean unit hypercube were generated to compare
performance on a smooth obstacle set with no possibility of axis
alignment between samples and obstacles. The setups for 2D and 3D
(Figures \ref{fig:sphere2} and \ref{fig:sphere3}) were fixed,
while in 4D, obstacles were randomly generated to match a specified
spatial coverage. Again, Halton and lattice sampling consistently outperformed random sampling, as shown in Table~\ref{table:summary}.

As an additional high-dimensional example, an 8D kinematic chain planning problem with rotational joints, eight links, and a fixed base was created (Figure \ref{fig:chain8}). The solution required the chain to be extracted from one opening and inserted into the other, as inspired by \citep{SRL-AY-SML:05}. The chain cost function was set as the sum of the absolute values of the angle differences over all links and the connection radius was thus scaled by $\sqrt{d}\pi$. With this high dimension and new cost function, the Halton and lattice still perform as well as or better than i.i.d. sampling (see Table
\ref{table:summary}).

Within the OMPL environment, rigid body planning problems from the
OMPL test banks were posed for SE(2) and SE(3) configuration
spaces. In the SE(2) case, one rotational and two translational
degrees of freedom are available, resulting in a three dimensional
space, shown in Figure \ref{fig:se2}. The
SE(3) problem consists of an ``L-shaped" robot moving with three
translational and three rotational degrees of freedom, resulting in
six total problem dimensions, shown in Figure
\ref{fig:se3}. As already mentioned, the rigid body planning problems
are not strictly covered by the theory in this paper, and thus the SE(2) and SE(3) lattices use the spreading method
described in Section \ref{sec:methods}. The results, summarized in
Table \ref{table:summary}, show that Halton and lattice sampling generally outperform i.i.d. random sampling. 

Lastly, planning problems with kinodynamic constraints were simulated
within the Julia environment using the gDPRM algorithm defined in
Section \ref{sec:dkp}, with reference to \citep{ES-LJ-MP:15b}. The connection radius was computed using
equation~\eqref{dprm-ed} in Theorem \ref{thrm:gDPRM}. First, a double-integrator model with two spatial dimensions and two velocity
dimensions was posed to simulate a system with drift (Figure
\ref{fig:doubleInt}). To obtain better spatial coverage, the
lattice was spread as described in Section \ref{sec:methods}. 
Second, results were simulated for the Reeds-Shepp
car system \citep{JAR-LAS:90}, a regular driftless control-affine
system with one rotational and two translational degrees of
freedom (Figure \ref{fig:doubleInt}). The Reeds-Shepp car is
constrained to move with a unit speed (forwards or backwards) and turn
with a fixed radius. Although these dynamics are different from those
discussed in Section~\ref{sec:dkp}, we can still define gDPRM with
respect to \citep{ES-LJ-MP:15a} and evaluate the benefits of
deterministic low-dispersion sampling. Again, spreading was used in the spatial dimension. With the addition of kinodynamic constraints, the Halton and lattice sampling continue to outperform i.i.d. random sampling, as summarized in Table \ref{table:summary}.


\begin{figure}[h!]
\center
	\subfigure[]{\label{fig:recMaze2} \includegraphics[width=0.28\textwidth]{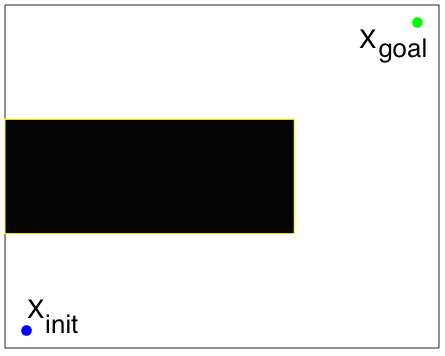}}
	\subfigure[]{\label{fig:recMaze3} \includegraphics[width=0.28\textwidth]{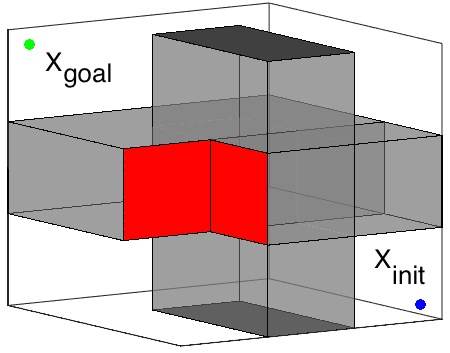}}
	\subfigure[]{\label{fig:sphere2} \includegraphics[width=.25\textwidth]{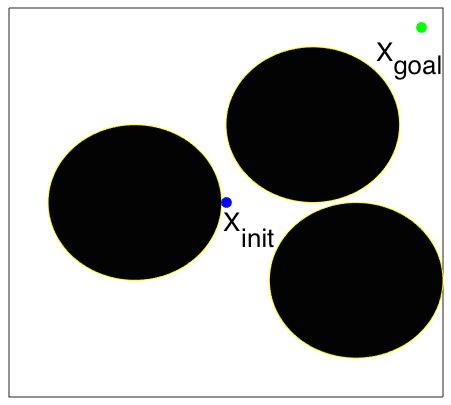}} 
    \subfigure[]{\label{fig:sphere3} \includegraphics[width=.3\textwidth]{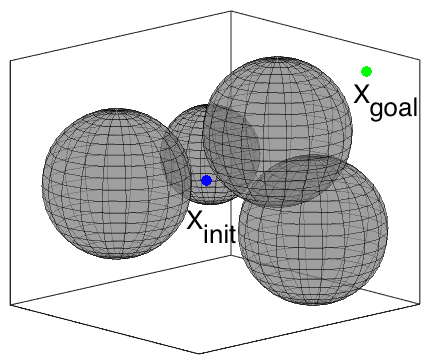}} 
    \subfigure[]{\label{fig:chain8} \includegraphics[width=.296\textwidth]{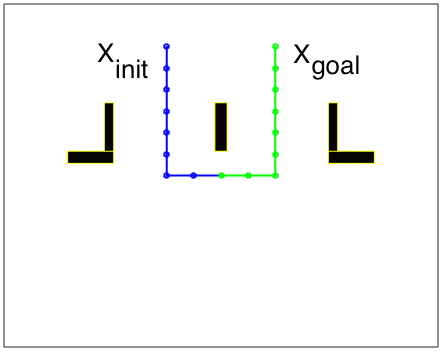}} 
    \subfigure[]{\label{fig:se2} \includegraphics[width=.23\textwidth]{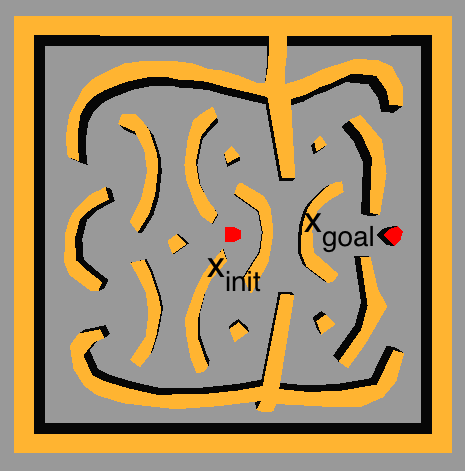}} 
    \subfigure[]{\label{fig:se3} \includegraphics[width=.28\textwidth]{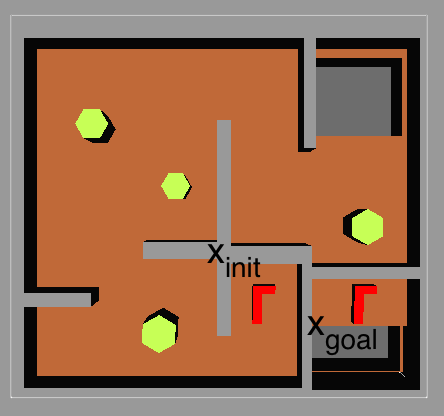}} 
    \subfigure[]{\label{fig:doubleInt} \includegraphics[width=.263\textwidth]{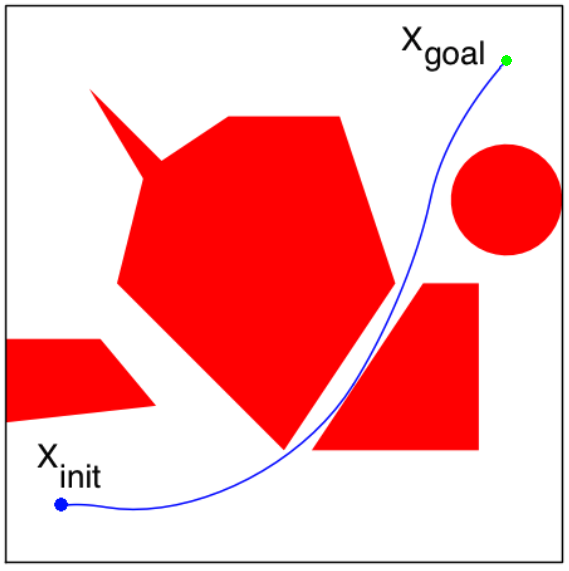}} 
\caption{Images of the recursive maze planning problem in 2D
  (\ref{fig:recMaze2}) and 3D (\ref{fig:recMaze3}) and the spherical
  obstacle sets in 2D (\ref{fig:sphere2}) and 3D
  (\ref{fig:sphere3}). Also shown are an 8D kinematic chain planning
  problem in \ref{fig:chain8} and the OMPL rigid body planning
  problems for SE(2) and SE(3) in \ref{fig:se2} and \ref{fig:se3}
  respectively. Lastly, \ref{fig:doubleInt} shows the setup for the 
  double-integrator and Reeds-Shepp car. Note that the problems in \ref{fig:doubleInt} have kinodynamic constraints, where we are not only considering straight line connections, but ones with curvature. A summary of results can be found in Table \ref{table:summary}.}
\label{fig:problems}
\end{figure}

\subsection{Summary of Results}\label{sec:sum}

Table \ref{table:summary} shows a summary of the results from
simulations detailed in Section \ref{sec:tests}. Results are shown normalized by the
i.i.d. sampling results. In each case the sample count at which
a success rate greater than 90\% is achieved and sustained is
reported. Additionally, the solution costs at a medium and high
sampling count are shown. For all cases the lattice sampling finds a
solution with fewer or an equal number of samples and of lower or
equal cost than that found by i.i.d. sampling. The Halton sampling
also always finds a solution at lower sample counts than i.i.d. sampling, and almost
always finds solutions of lower cost as well. The deterministic low-dispersion
sequences particularly outperform random sampling in terms of number
of samples required for a 90\% success rate.

\begin{table}[h!]
\begin{center}	
	\begin{tabular}{|c|c|c@{\hskip 0.1in}|c|c|c|c@{\hskip 0.1in}|c|c|c|}
	\multicolumn{3}{c}{} & \multicolumn{3}{c}{Halton} & \multicolumn{1}{c}{} & \multicolumn{3}{c}{Lattice} \\
	\cline{1-2} \cline{4-6} \cline{8-10}
	Dim & Obstacles && 90\% Success & Medium & High && 90\% Success & Medium & High \\
	\hline \hline
2 & Rectangular && 38\% & 118\% & 80\% && 15\% & 56\% & 80\% \\
3 & Rectangular && 36\% & 88\% & 94\% && 19\% & 80\% & 87\% \\
2 & Rect Maze && 13\% & 98\% & 99\% && 13\% & 100\% & 99\% \\
\hline \hline
2 & Sphere && 16\% & 93\% & 99\% && 7\% & 93\% & 99\% \\
3 & Sphere && 36\% & 97\% & 100\% && 8\% & 97\% & 99\% \\
4 & Sphere && 100\% & 97\% & 97\% && 100\% & 97\% & 100\% \\
\hline \hline
2 & Recursive Maze && 33\% & 100\% & 100\% && 18\% & 100\% & 100\% \\
3 & Recursive Maze && 22\% & 95\% & 99\% && 22\% & 96\% & 98\% \\
4 & Recursive Maze && 56\% & 95\% & 98\% && 56\% & 100\% & 100\% \\
5 & Recursive Maze && 45\% & 97\% & 96\% && 60\% & 95\% & 96\% \\
6 & Recursive Maze && 56\% & 95\% & 97\% && 75\% & 94\% & 96\% \\
8 & Recursive Maze && 56\% & 98\% & 99\% && 75\% & 99\% & 99\% \\
\hline \hline
8 & Chain && 67\% & 112\% & 91\% && 7\% & 76\% & 87\% \\
\hline \hline
3 & SE(2) && 81\% & 96\% & 100\% && 81\% & 101\% & 101\% \\
6 & SE(3) && 32\% & 96\% & 93\% && 42\% & 94\% & 95\% \\
\hline \hline
4 & Double-Integrator && 53\% & 90\% & 93\% && 30\% & 92\% & 96\% \\
3 & Reeds-Shepp Car && 44\% & 97\% & 99\% && 44\% & 98\% & 100\% \\
	\hline
	\end{tabular}
\end{center}
\caption{Summary of results. Each entry is divided by the results of i.i.d. sampling (averaged over 50 runs). For Halton sampling and lattice sampling, the number of samples at which 90\% success is achieved and the cost at a medium number of samples (near 700) and a high number of samples are shown (highest samples simulated, always 3000 or greater). Note that nearly all table entries are below 100\%, meaning the Halton and lattice sampling outperformed i.i.d. sampling.}
\label{table:summary}
\end{table}

\subsection{Nondeterministic Sampling Sequences}\label{sec:nondet}

The above simulations showed deterministic lattice sampling, with a
fixed rotation around each axis, and the deterministic Halton
sequence outperform uniform i.i.d. sampling. Both deterministic
sequences have low $\ell_2$-dispersions of $O(n^{-1/d})$, but sequences
with the same order $\ell_2$-dispersion need not be deterministic. Figure
\ref{fig:nondet} shows results for a randomly rotated and randomly
offset version of the lattice (again, the $\ell_2$-dispersion and
neighborhoods are all still deterministically the same). The same cases in Table
\ref{table:summary} were run for the randomly rotated lattice and the
results showed it performed as well as or better than random sampling
(over 50 runs). In general, low-dispersion random sequences might provide some advantages, e.g., eliminating axis alignment issues while still enjoying deterministic guarantees (see Section \ref{sec:theory}). Their further study represents an interesting direction for future research.

\begin{figure}[h!]
\center
	\subfigure[]{\label{fig:maze_cost_nonDet} \includegraphics[width=0.3\textwidth]{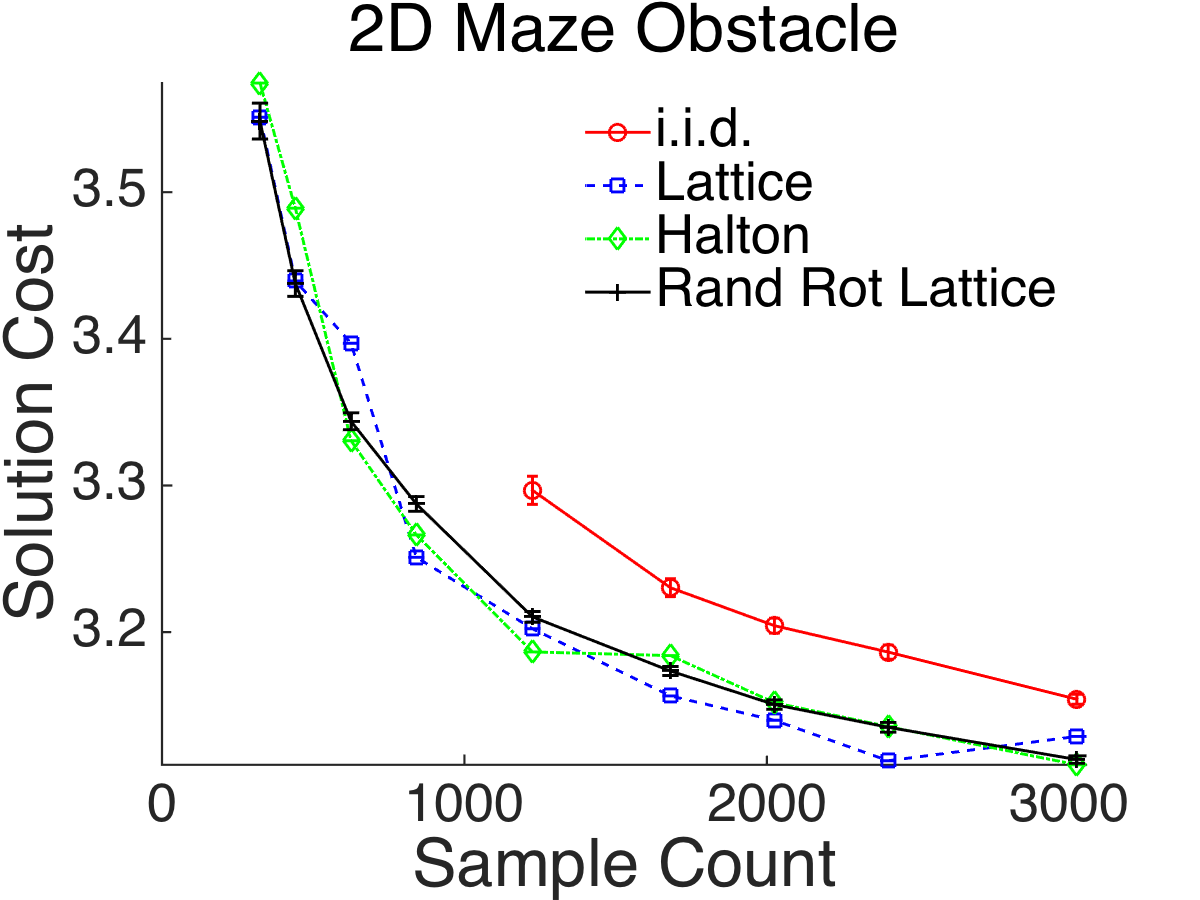}}
    \subfigure[]{\label{fig:maze_success_nonDet} \includegraphics[width=.3\textwidth]{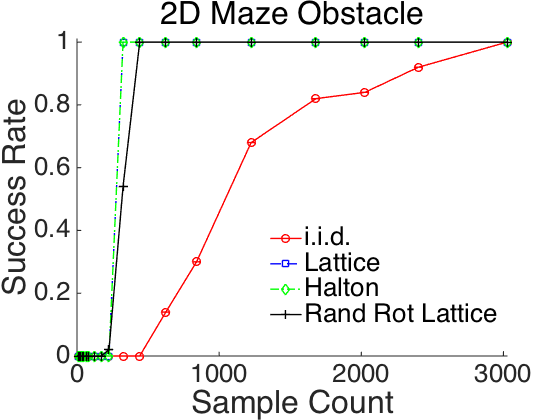}} 
    \subfigure[]{\label{fig:se3_cost_nonDet} \includegraphics[width=.3\textwidth]{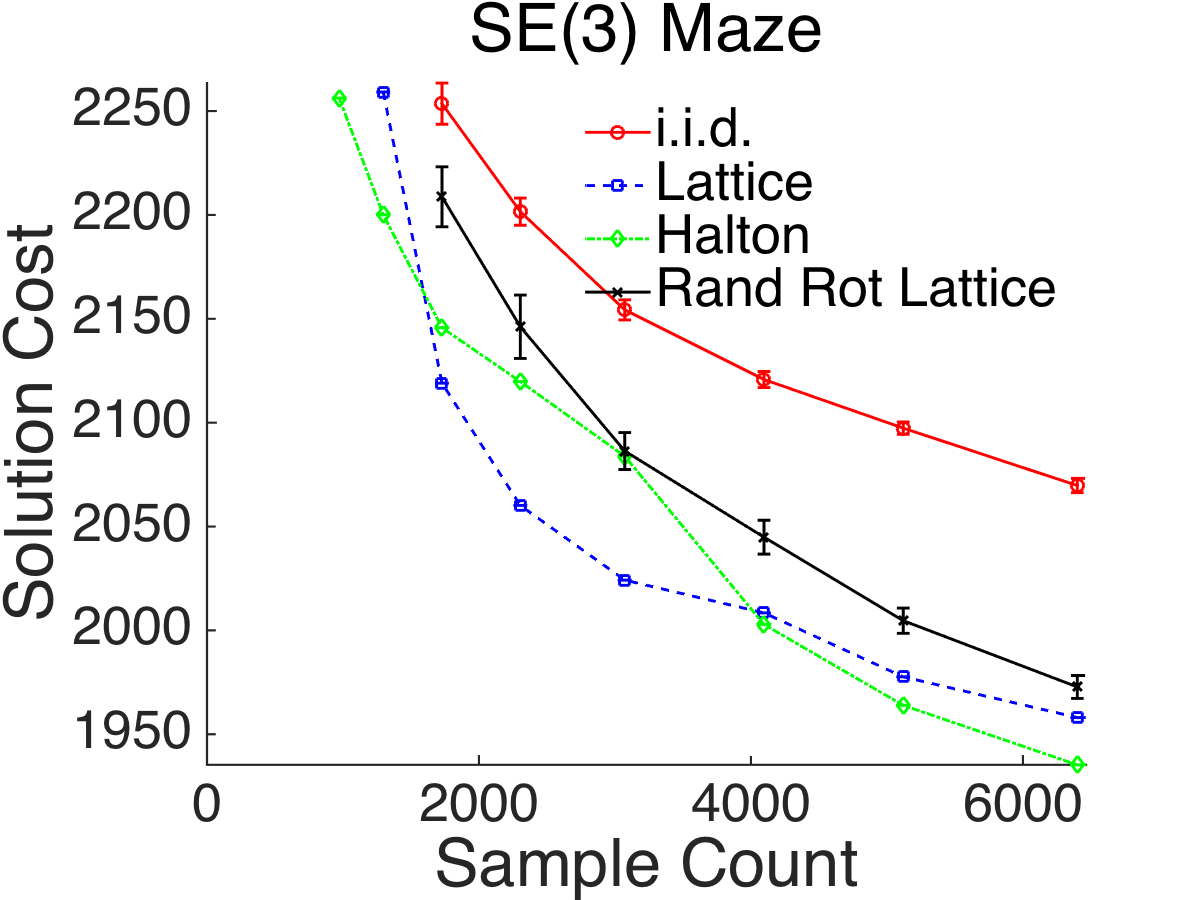}} 
\caption{Results for deterministic and nondeterministic low-dispersion sampling. ``Rand Rot Lattice'' refers to a randomly rotated lattice.}
\label{fig:nondet}
\end{figure}

\section{Conclusions}\label{sec:conc}
This paper has shown that using low-dispersion
sampling strategies (in particular, deterministic) can provide substantial benefits for solving the
optimal path planning problem with sampling-based algorithms, in terms of deterministic performance guarantees, reduced computational complexity per given number of samples, and superior practical performance.

This paper opens several directions for future research. First, we plan to deepen our study of deterministic kinodynamic motion planning, in particular in terms of tailored notions of 
sampling sequences and dispersion and more general dynamical models. Second, it is of interest to extend the results herein to
other classes of sampling-based motion planning algorithms (beyond the ones studied in this paper), especially
the large class of \emph{anytime} algorithms (e.g., RRT/\RRTstar\!). This leads directly into
a third key direction, which is to study alternative low-dispersion
sampling strategies beyond the few considered here, particularly
\emph{incremental} sequences for use in anytime algorithms. There is
already some work in this area, although thus far it has focused on the use of
such sequences for the feasibility problem \citep{AY-SML:04,
  SRL-AY-SML:05, AY-SJ-SML-ea:09}. It may also be of interest to study
low-dispersion sampling strategies that incorporate prior knowledge of
the problem by sampling non-uniformly, as discussed in Section \ref{sec:nus}. Fourth, we
plan to investigate the topological relationship
between the optimal path cost and that of the best
strong-$\delta$-clear path, in order to frame the convergence rate in
terms of the true optimal cost. Fifth, from a practical standpoint, it is of interest to adapt existing algorithms or design new ones that explicitly leverage the structure of low-dispersion sequences (e.g., fast nearest neighbor indexing or precomputed data structures). This would be especially beneficial in the domain of kinodynamic motion planning. Finally, leveraging our convergence rate results, we plan to further investigate the issue of certification for sampling-based planners, e.g., in the context of trajectory planning for drones or self-driving cars.

\section*{Acknowledgements}
This work was supported  by NASA under the Space Technology Research Grants Program, Grant NNX12AQ43G.  Lucas Janson was partially supported by NIH training grant T32GM096982. Brian Ichter was supported by the Department of Defense (DoD) through the National Defense Science \& Engineering Graduate Fellowship (NDSEG) Program.

\bibliographystyle{abbrvnat}
\bibliography{../../../bib/alias,../../../bib/main}

\end{document}